\documentclass[a4paper,11pt]{amsart}
\usepackage{style} 

\usepackage[hyperref,doi=false,url=false,isbn=false, giveninits=true, backref=true,style=alphabetic,maxbibnames=99, backend=bibtex, natbib=true]{biblatex}
\usepackage{subcaption}

\title{Quantitative Gaussian-Process Limits of Tensor Programs}
\author{Andrea Agazzi}
\address{Institute for Mathematical Statistics and Actuarial Sciences, University of Bern, Switzerland}
\email{andrea.agazzi@unibe.ch}
\author{Eloy Mosig Garc\'\i a}
\address{Department of Mathematics, University of Pisa, Italy}
\email{eloy.mosig@phd.unipi.it}
\author{Dario Trevisan}
\address{Department of Mathematics, University of Pisa, Italy}
\email{dario.trevisan@unipi.it}
\date{\today}

\begin{document}

\begin{abstract}
  We study the infinite-width Gaussian-process limit of random neural networks
  through the lens of tensor programs, and we provide a quantitative convergence
  theory in Wasserstein distance.
  Our main result gives explicit finite-width error bounds, of order inverse square-root of the widths
  between finite-network executions and their
  Gaussian-process limits.   The framework is architecture-agnostic and covers feed-forward models together
  with weight-sharing schemes relevant for recurrent and transformer-type
  architectures.\\[-25pt]
\end{abstract}

\maketitle

\section{Introduction}

A depth-$M$ Multilayer Perceptron (MLP) with input $x\in\R^{d}$ and layer widths $(n_0,\ldots,n_{M})$ 
 is defined as the map $h^{(M)}~:~\R^{n_0}\to \R^{n_M}$, where $h^{(0)}(x) := x$ and
\begin{align}
  \label{eq:intro-ffnet}
  \hfill\qquad \qquad h^{(\ell+1)}(x) &:= \phi_\ell\bra{W^{(\ell)} h^{(\ell)}(x)},
  \qquad \text{for } \ell=0,\ldots,M-1,
\end{align}
with activation functions  $\phi_\ell:\R\to\R$ acting componentwise on their input, and parameters (or weights) $W^{(\ell)}\in\R^{n_{\ell+1}\times n_\ell}$ for 
$\ell \in  \{0, \dots M-1 \}$. These parameters are typically initialized 
randomly, drawn independently  from a common (layer-dependent) distribution. A standard practice is to choose such distribution as Gaussian centered with $\mathcal O(n_{\ell}^{-1})$ variance, for instance
\begin{equation}\label{e:init}
W^{(\ell)}_{ij} \sim_{iid} \mathcal N(0, n_\ell^{-1})\,.
\end{equation}
Under this choice of scaling, when the widths $n_1,\ldots,n_{M-1}$ go to infinity, 
the pre-activations $g^{(\ell)}:=W^{(\ell)}h^{(\ell)}$ converge
to a centered Gaussian process whose kernel is given by an explicit recursion
in $\phi$ and in the input Gram matrix. This is the
\emph{neural-network Gaussian-process} (NNGP) limit, first identified by
\cite{neal96} for $M = 2$ and later studied for generic integer $M \geq 2$ by
\cite{Lee2020}. Quantitative versions of these Central Limit Theorem (CLT) results were obtained in 
\cite{BasteriTrevisan2023,Favaro2025, trevisan2023wide} and more recently in \cite{celli2026wide, giovagnini2026universalitydeepneuralnetworks}. This work extends such quantitative convergence results to a significantly richer class of functions called \emph{tensor programs} defined below.

\subsection{Tensor programs} 
As discussed in \cite{yang1}, the above function class, as well as many other neural network architectures, can be interpreted as a specific instance of abstract algorithms, called $\netsor$ programs. These are a subclass of the more general \emph{tensor programs} \cite{yang2,yang3}, consisting of a sequence of lines where variables are declared and then combined via elementary operations. In particular, a $\netsor$ program $T$ is defined by specifying its \emph{inputs} and its \emph{operations} as follows:

\subsubsection*{Inputs} The inputs\footnote{a set of variables that are declared at the beginning of the program, and whose definition does not rely on other variables.} of a {\netsor} program are associated to two different variable types:
\begin{itemize}
  \item {\bf (Input) $\sH$-vars} 
are a collection of fixed real vectors
\[
\cX:=\{x_k\}_{k} \qquad x_k \in \mathbb R^{d_k}
\]
generalizing the input data on which the network acts. The dimensions of this collection are stored in set $\nin := \{d_k\}_k$. These variables belong to the class of $\sH$-vars, i.e., real, vector-valued quantities interpretable as hidden states of the network.
\item {\bf $\sA$-vars} denote an abstract collection of matrix-valued parameters 
$$\sW := \{\sW^{(j)}\}_{j },$$
representing the weights (and possibly biases)  of the network. 
\end{itemize}
Conforming to the notation in \cite{yang1}, we recall the type and dimension of a newly declared variable through the $\colon$ notation\footnote{This notation is borrowed from the variable declaration syntax in computer science and should be thought of as an alternative to the symbol $\in$.}, so that any $\sA$-var $\sW$ of dimension $n \times m$ will be declared as $\sW:\sA(n,m)$, while $\sH$-vars (and $\sG$-vars defined below)
of dimension $n$ will be respectively  introduced as 
$\sh:\sH(n)$ and $\sg : \sG(n)$. 

\subsubsection*{Operations} The variables defined above can be combined through the following operations 
\begin{itemize}
  \item $\matmul$:  The Matrix-vector multiplication operation combines a $\sW : \sA(n, m)$ and a $\sh : \sH(m)$. The result of this operation is a new $n$-dimensional vector-valued $\sH$-var behaving as a Gaussian when the dimension being contracted diverges. Denoting the subset of $\sH$-vars satisfying this property as $\sG$-vars (abstracting the concept of pre-activations), we write an instance of this operation as
$$\matmul \quad \sg := \sW \sh \colon  \sG(n).$$
\item $\nonlin$: Componentwise composition of $\sh_1, \ldots, \sh_m : \sH(n)$ with a (possibly nonlinear) function $\phi: \R^m \to \R$, yielding a new $\sH$-var of dimension $n$. This operation is declared as
\[ \nonlin \quad \sh := \phi(\sh_1, \ldots, \sh_m) \colon \sH(n)\]
\end{itemize}
%
%
The execution of each of the above operations yields a \emph{line} of the {\netsor} program. We denote by $L\in \mathbb N$ the total number of lines in a program.

\subsubsection*{Hidden widths and structural constants} 
The hidden widths of the program are defined as the dimensions over which the program performs internal matrix contractions. More precisely, a dimension $n$ is called a \emph{hidden width} if there exists a $\sG$-var of shape $n$ which is produced by a $\matmul$ line and whose downstream descendants in the graph are later used as the parent of another $\matmul$ line. Equivalently, $n$ is internal to a chain of matrix multiplications. We denote by $\nhidd$ the collection of hidden widths, which will be considered as asymptotically large in this work. By contrast, recall that input dimensions $\nin$ are kept fixed. Dimensions that are neither input nor hidden may be referred to as \emph{output dimensions} $\nout$. 

Furthermore, quantities that remain fixed throughout the asymptotic limit, such as the input space $\mathcal X$, the input dimensions $\nin$, the activations $\phi$ appearing in $\nonlin$ lines, and the order of the Wasserstein distance, are collectively referred to as \emph{structural constants}. We then use $\scc$ to denote a generic bounding constant that depends exclusively on these structural parameters and is independent or uniform in the hidden widths $\nhidd$. To keep notation simple, we conventionally allow the value of $\scc$ to change from line to line.

\begin{example}[Shallow neural network]\label{example:shallow} The following $3$-line {\netsor} program yields \eqref{eq:intro-ffnet} with $M =2$.  In this case, $\nin = \{d\}$, $\nhidd = \{n\}$, $\nout = \{d_{\textrm{out}}\}$.
\begin{algorithm}
\floatname{algorithm}{$\netsor$ program}
\caption{Shallow neural network evaluated on a single input}
    \label{alg:shallow}
\begin{algorithmic}
\Require $\sx: \sH(d)$
\Require $\sW^{(0)}: \sA(n, d)$
\Require $\sW^{(1)}: \sA(d_{\textrm{out}},n)$
\State \algline $\sg^{(1)} = \sW^{(0)} \sx: \sG(n)$ \Comment{\matmul}
\State \algline $\sh^{(1)} = \phi( \sg^{(1)}): \sH(n)$ \Comment{\nonlin}
\State \algline $\sg^{(2)} = \sW^{(1)} \sh^{(1)}: \sG(d_{\textrm{out}})$ \Comment{\matmul}
\end{algorithmic}
\end{algorithm}
\end{example}

\subsubsection*{Program graphs}
To better visualize the above concepts, it is useful to associate to a {\netsor} program a directed graph, encoding the flow of variables through the program lines. The nodes of the graph are the $\sH$-vars (hence also the $\sG$-vars) labelled by their type and shape. A {\matmul} line $\sg=\sW\sh$ contributes a solid directed edge $\sh\longrightarrow \sg$ labelled by the corresponding $\sA$-var $\sW$. Thus, repeated use of the same $\sA$-var is represented by repeated occurrences of the same edge label. A {\nonlin} line $\sh=\phi(\sh_1,\ldots,\sh_k)$ contributes directed edges $\sh_i \longrightarrow \sh$ which we draw as dashed lines with $\phi$ as label.

\begin{figure}[h]
\centering
\begin{tikzpicture}[
  node distance=1.85cm,
  var/.style={rectangle, draw, rounded corners, minimum height=7mm, inner sep=4pt},
  edge/.style={->, thick, >=Latex},
  nonlin/.style={->, thick, dashed, >=Latex}
]
\node[var] (x) {$\sx:\sH(d)$};
\node[var, right=of x] (g1) {$\sg^{(1)}:\sG(n)$};
\node[var, right=of g1] (h1) {$\sh^{(1)}:\sH(n)$};
\node[var, right=of h1] (g2) {$\sg^{(2)}:\sG(d_{\rm out})$};

\draw[edge] (x) -- node[above] {$\sW^{(0)}$} (g1);
\draw[nonlin] (g1) -- node[above] {$\phi$} (h1);
\draw[edge] (h1) -- node[above] {$\sW^{(1)}$} (g2);

\node[below=0.35cm of x] {\scriptsize input};
\node[below=0.35cm of h1] {\scriptsize hidden width $n$};
\node[below=0.35cm of g2] {\scriptsize output};
\end{tikzpicture}
\caption{Program graph of a shallow fully connected network evaluated on a single input. The input dimension is $d\in\nin$, the hidden width is $n\in\nhidd$, and the output dimension is $d_{\rm out}\in\nout$.}

\label{fig:program-graph-shallow}
\end{figure}
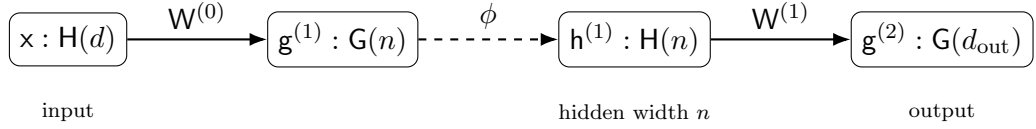

 \subsubsection*{Program execution} For a given choice of inputs $\mathcal X$, 
 the \emph{execution} of a program consists in the realization of its $\sH$-vars $\sh^{(j)}$ (including its $\sG$-vars $\sg^{(j)}$) as vector-valued random variables, with a given distribution on a common probability space. 

At the level of generality considered so far, the execution of a $\netsor$ program and the statistical properties of its $\sH$-vars can be interpreted in different ways.  In this work, we focus on the following two execution rules. In both executions, $\nonlin$ is intended as the component-wise application of $\phi$, hence they differ only on the concrete implementations of the {\netsor} operation.
 \begin{itemize}
  \item \emph{Finite-width}: this execution corresponds to the classical evaluation of the feed-forward pass of a neural network, with Gaussian initialization of its weights. For every input $\sA$-var $\sW:\sA(n,m)$, we realize $W_{ij}\sim \mathcal{N}(0,1/m)$, all independently of each other. Then, given a $\sH$-var $h \in \R^{m}$ and an $\sA$-var $W \in \mathbb \R^{n \times m}$, the $\matmul$ line $\sg = \sW \sh$ is interpreted as a standard matrix-vector multiplication, defining a random vector $g = W h$ with values in $\R^{n}$. 
\item \emph{Infinite-width}: this execution  replaces the finite-width matrix multiplications by their Gaussian central-limit approximation, and interprets $\sG$-vars $\bar g^{(1)}, \dots, \bar g^{(G)}$ appearing in the program as \emph{jointly Gaussian} random variables. Precisely,  the output of a $\matmul$ line $\sg := \sW \sh$ with $\sW:\sA(n,m)$ and $\sh:\sH(m)$
is defined as Gaussian random variable $\bar g$ taking values in $\R^n$, such that, for any previously introduced $\sG$-var $\sg'$ with corresponding Gaussian $\bar g'$ (including the case $\sg' = \sg$), the cross-covariance $\Sigma_{\bar g \bar g'}$ is defined as 
\begin{equation}
\label{eq:sigma-gg-mf}
\Sigma_{\bar g\bar g'}:=
\begin{cases}
0\qquad &
\text{if $\sg' = \sW' \sh'$ for an $\sA$-var $\sW'\neq \sW$}
\\
\bar{ h}\ldot \bar{h}' \Id_n \qquad &  
\text{if $\sg' = \sW \sh'$, with same $\sA$-var $\sW$ and $\sh':\sH(m)$}
\end{cases}
\end{equation}
where for two executions $\bar h, \bar h' \in \mathbb R^m$  of $\sH$-vars we introduce the \emph{infinite-width} product notation:
\begin{equation}\label{eq:infinite-width-product}  \bar{ h}\ldot \bar{h}' := \frac1m\,\EE\sqa{\bar h^\top \bar h'},\end{equation}
where with $\top$ we denote the matrix transpose.
\end{itemize}

\subsection{Main result} The two executions presented above are connected by our main result, presenting quantitative distributional convergence estimates in the wide limit 
$$ \min_{n \in \nhidd} n \to \infty\,,$$
 for the finite-width execution $h^{(\ell)}$ of a generic {\netsor} program to its infinite-width counterpart $\bar h^{(\ell)}$. The convergence is established in Wasserstein distance, defined for any $p \geq 1$ for a pair of probability distributions $\mu, \nu$ with finite $p$-th moments as
$$\mathcal W_p(\mu, \nu)^p := \inf_{\gamma \in \Gamma(\mu, \nu)}\int \|x-y\|^p \mathrm{d} \gamma(x,y)$$
where $\|\cdot\|$ denotes the Euclidean norm and $\Gamma(\mu,\nu)$ is the set of all couplings of $\mu$ and $\nu$. By slight abuse of notation, for a pair of random vectors $X,Y$ with distributions $\mathbb P_X, \mathbb P_Y$ we write $\mathcal W_p(X, Y) := \mathcal W_p(\mathbb P_X, \mathbb P_Y)$.
\begin{theorem}[Quantitative CLT]
\label{thm:main}
Consider a $\netsor$ program with $L$ lines, denote the $\sH$-vars defined by the program with $\sh^{(\ell)}:\sH(n_\ell)$, and write 
$$
\bigl(h^{(\ell)}\bigr)_{\ell=1,\ldots,L},
\qquad
\bigl(\bar h^{(\ell)}\bigr)_{\ell=1,\ldots,L},
$$
for their finite-width and infinite-width executions. Assume that all functions appearing in {\nonlin} lines are Lipschitz continuous.
 Then, for every $p\ge1$, there exists a
constant $\scc<\infty$, depending only on $p$ and on the structural constants of the
program, such that
\begin{equation}\label{eq:inequality-theorem-main-nngp}
\W_p\left(
\left(\frac{h^{(\ell)}}{\sqrt{n_\ell}}\right)_{\ell=1,\ldots,L},
\left(\frac{\bar h^{(\ell)}}{\sqrt{n_\ell}}\right)_{\ell=1,\ldots,L}
\right)
\le
\scc
\sum_{m\in\nhidd}\frac1{\sqrt m}.
\end{equation}
\end{theorem}
Our second main result can be regarded as a law of large numbers (LLN) for empirical kernels built from $\sH$-vars, and follows from Theorem~\ref{thm:main}, but it also provides a key ingredient in the induction argument for its proof.

\begin{corollary}[Kernel LLN]
\label{coro:kernel-lln}
Assume the hypotheses of Theorem~\ref{thm:main}. Then for any pair of $\sH$-vars
$\sh,\sh':\sH(n)$,
\[
\left\|
\frac1n h^\top h'
-
\bar h\ldot\bar h'
\right\|_{L^p}
\le
\scc
\sum_{m\in\nhidd\cup\nout}\frac1{\sqrt m}.
\]
where $\scc<\infty$ depends only on $p$ and on the structural constants of the program.
\end{corollary}

We remark that the set $\nout$ appears in the right-hand side only when one takes empirical kernels over output dimensions, i.e., $n \in \nout$. 

\subsubsection*{Scope and comparison with existing limits.}
Let us emphasize that Theorem~\ref{thm:main} is not merely a reformulation of existing
quantitative central limit theorems for fully connected feed-forward networks. The tensor
program formulation allows one to treat architectures in which the same weight matrix is
used several times in the computation, possibly across different layers or different input
branches. This includes, for example, weight-sharing patterns arising in recurrent networks
and in the joint evaluation of a network on several inputs (see the examples proposed in Subsection \ref{subsec:program-graphs}). In such cases, the relevant
finite-width dependencies are not captured by a layer-by-layer independent CLT, and the
conditional Gaussian structure of tensor programs becomes essential.

The result should also be compared with Yang's master theorem \cite[Theorem 5.4]{yang1} for tensor programs. That theorem gives a qualitative law-of-large-numbers description of the infinite-width limit for a very broad class of programs and nonlinearities. Our theorem is more restrictive in its regularity assumptions, since we require Lipschitz continuity, but it gives an explicit quantitative Wasserstein estimate between
the finite-width execution and the infinite-width Gaussian-process execution. In this sense, the present result can be viewed as a quantitative counterpart, for a restricted, but architecture-rich, class of tensor programs, to the qualitative tensor-program limit theory.

Finally, the basic {\netsor} language does not cover all operations needed for modern
architectures, such as transformers. Subsection~\ref{subsec:extensions-kernel} introduces a
restricted extension, denoted $\netsork$, in which scalar kernel variables and
scalar-parametric nonlinearities are allowed. This extension is designed to cover
attention-type layers while keeping the quantitative proof based on kernel LLN estimates and finite-dimensional stability bounds.

\subsubsection*{Structure of the proof}

The proof of Theorem~\ref{thm:main} is a line-by-line induction on the program.
One mild subtlety is that the kernel LLN in Corollary~\ref{coro:kernel-lln} is both a
consequence of the quantitative convergence theorem and an ingredient in its proof. This
does not lead to a circular argument: at the induction step from lines $1,\ldots,r-1$ to
line $r$, we use the theorem only for the truncated program up to line $r-1$. This gives
the quantitative coupling of all previously constructed $\sH$-vars. Combining this coupling
with a standard quantitative law of large numbers for the independent coordinates of the
infinite-width execution yields the kernel LLN for all kernels that can appear before line
$r$.

Once this kernel control is available, the induction step is as follows. If line $r$ is a
\textsc{NonLin} instruction, the estimate is propagated directly by the Lipschitz stability
of the activation. If line $r$ is a {\matmul} instruction, both the finite-width and
infinite-width executions are conditionally Gaussian given the previous $\sH$-vars. We
couple these two conditional Gaussian laws by using the same standard Gaussian noise.
The resulting error is reduced to estimating the difference between the corresponding
conditional means and square-roots of conditional covariances.

The structure lemmas \ref{lem:structure-lemma-finite-width} and \ref{lem:structure-lemma-infinite-width} show that these conditional means and covariances are finite-dimensional
functions of the Gram matrices of the parent $\sH$-vars associated with the same
$\sA$-var. After the non-degeneracy reduction, these functions are differentiable at
the infinite-width Gram point. The kernel LLN for the truncated program therefore controls
the difference between the finite-width and infinite-width conditional parameters. This
closes the induction and proves Theorem~\ref{thm:main}. Applying the same argument to
the full program then gives Corollary~\ref{coro:kernel-lln}.



\subsubsection*{Related works}

The convergence of neural networks to Gaussian Processes in the wide limit was first established for shallow, fully connected neural networks by \cite{neal96} and subsequently extended to deep fully connected networks \cite{matthews18}. These qualitative results, coupled with the introduction of the Neural Tangent Kernel (NTK) \cite{jacot, arora}, provided a rigorous foundation for analyzing the training and generalization of infinite-width networks.

To systematically encompass the vast landscape of modern architectures, the Tensor Programs formalism was introduced \cite{yang0, yang1, yang2, yang3}. This framework provides a universal algebraic language to prove that the forward and backward passes of essentially any standard architecture converge to Gaussian Process and NTK limits. Importantly, the Tensor Programs formalism also clarifies the distinction between the NTK regime and the feature learning regime \cite{yang4}, where representations actively evolve during training. In this work, we focus on the former setting.

Parallel to the infinite-width asymptotic, the infinite-depth and joint wide-and-deep limits have been extensively explored to understand hierarchical feature propagation. For instance, \cite{peluchetti20a} demonstrated that residual networks converge to diffusion processes in the infinite-depth limit, while \cite{hayou} formalized the commutativity of sequential width and depth limits. More recently, \cite{agazzi2026, borjan} established quantitative scaling limits for residual attention architectures in the wide, deep, and particle mean-field limits.

While foundational results established qualitative asymptotic limits, determining the precise rate of convergence at finite width has become a central focus. From a rigorous probabilistic standpoint, quantitative CLTs characterizing the convergence of finite-width network architectures to their Gaussian wide limits were established for fully connected networks in \cite{BasteriTrevisan2023, trevisan2023wide, bordino25, Favaro2025} by using optimal transport theory in the first two cases, second-order Poincaré inequalities in the third, and Stein-Malliavin method in the last. Notably, also functional quantitative CLTs were proven in \cite{eldan21, cammarota23} and extended to deep networks and strengthened in \cite{Favaro2025}. Moving beyond initialization, \cite{mosig26} extended the optimal transport quantitative bounds from \cite{BasteriTrevisan2023} to positive training time in the shallow MLP case. 

Very recently, these quantitative guarantees have been significantly broadened. \cite{celli2026wide} provided convergence rates with an explicit dependence on general weight distributions and hyperparameters, while \cite{giovagnini2026universalitydeepneuralnetworks} bypassed traditional Gaussian assumptions to prove universality in deep networks via the Lindeberg exchange principle. Concurrently, efforts have been made to quantitatively analyze more intricate architectures. Prior to these quantitative efforts, \cite{hron20} rigorously derived the exact NNGP and NTK kernels for deep attention networks in both $d^{-1}$ and $d^{-1/2}$ scalings. Building upon this, \cite{sakai2026infinitewidthlimitsingleattention} recently leveraged the Tensor Programs framework to analyze the infinite-width limit of attention layers. Notably, they extended the Master Theorem in \cite{yang1} to attention architectures with the $d^{-1/2}$ scaling in the dot product, as opposed to the $d^{-1}$ scaling used by \cite{yang1}.

Finally, we mention a different but related direction, concerning Gaussian-process limits for quantum neural networks. In that setting, the random function generated by a parametrized quantum circuit is typically an expectation value of observables over many qubits, and Gaussian limits arise in a large-width or many-qubit regime under suitable locality or weak-correlation assumptions. Recent work \cite{girardi2025trained, melchor2025quantitative} has shown that both untrained and trained quantum neural networks may converge to Gaussian processes, and that this convergence can be quantified in Wasserstein distance. Although the mechanisms and architectures differ substantially from the tensor-program setting considered here, these results share the same general objective: to turn qualitative Gaussian-process limits for random neural models into explicit finite-size approximation bounds.

\subsubsection*{Structure of this work}  Section~\ref{sec:examples_extensions} illustrates the scope of the main theorem through several examples, including shallow and deep fully connected networks, recurrent architectures, and its extension to $\netsork$ programs, which adds scalar kernel variables and scalar-parametric nonlinearities. 
Section~\ref{sec:tech} develops the structural ingredients of the argument: conditional-law formulas for finite- and infinite-width executions, and the reduction to non-degenerate representative programs. Section~\ref{sec:proof-main} proves the main results
for both base $\netsor$ programs and extended $\netsork$ programs.  
Finally, Section~\ref{sec:numerical_experiments} presents numerical experiments illustrating finite-width convergence to the Gaussian-process limit across several architectures.

\section{Examples and Extensions}
\label{sec:examples_extensions}

Before we move to the proof of our main results, we illustrate how they  apply to the forward pass of a large family of architectures.

\subsection{Examples} We start with simple examples, moving towards more complex ones.
\label{subsec:program-graphs}

\subsubsection*{Shallow fully connected network}
The previously defined $\netsor$ program \ref{alg:shallow} showcases a single-hidden layer neural network. The corresponding graph is shown in Figure \ref{fig:program-graph-shallow-graph}.
In this case, Theorem \ref{thm:main} yields:
\begin{equation}
\W_p\left(\left(\frac{h^{(1)}}{\sqrt{n}},\frac{g^{(2)}}{\sqrt{d_{\rm out}}}\right),
\,\left(\frac{\bar h^{(1)}}{\sqrt n}, \frac{\bar g^{(2)}}{\sqrt{d_{\rm out}}}\right) \right)\le \frac\scc{\sqrt n},
\end{equation}
where 
$\bar h^{(1)} / \sqrt{n}$ is the scaled componentwise image via $\phi$ of the random vector $\bar g^{(1)}$ with entries i.i.d. centered Gaussian variables with variance equal to $x^\top x$
and $\mathcal G = \bar g^{(2)} / \sqrt{d_{\rm out}}$ is the Neural Network Gaussian process (NNGP) for a shallow network \cite{jacot,Lee2020,BasteriTrevisan2023}.
$\bar g^{(1)}$ and $\bar g^{(2)}$ are to be interpreted as centered Gaussian variables with covariance as in Equation \eqref{eq:sigma-gg-mf}.

\begin{figure}[h]
\centering
\begin{tikzpicture}[
  node distance=1.85cm,
  var/.style={rectangle, draw, rounded corners, minimum height=7mm, inner sep=4pt},
  edge/.style={->, thick, >=Latex},
  nonlin/.style={->, thick, dashed, >=Latex}
]
\node[var] (x) {$\sx:\sH(d)$};
\node[var, right=of x] (g1) {$\sg^{(1)}:\sG(n)$};
\node[var, right=of g1] (h1) {$\sh^{(1)}:\sH(n)$};
\node[var, right=of h1] (g2) {$\sg^{(2)}:\sG(d_{\rm out})$};

\draw[edge] (x) -- node[above] {$\sW^{(0)}$} (g1);
\draw[nonlin] (g1) -- node[above] {$\phi$} (h1);
\draw[edge] (h1) -- node[above] {$\sW^{(1)}$} (g2);

\node[below=0.35cm of x] {\scriptsize input};
\node[below=0.35cm of h1] {\scriptsize hidden width $n$};
\node[below=0.35cm of g2] {\scriptsize output};
\end{tikzpicture}
\caption{Program graph of a shallow fully connected network. The dimension $n$ is a
hidden width because it is produced by the first $\matmul$ line and later contracted
by the readout $\matmul$ line. The input dimension $d$ and output dimension
$d_{\rm out}$ are structural.}

\label{fig:program-graph-shallow-graph}
\end{figure}
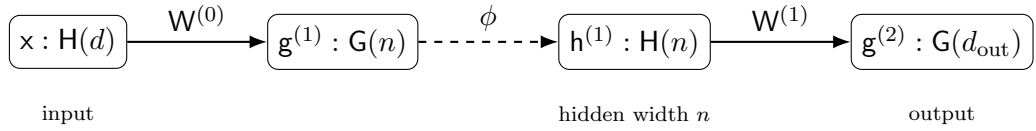

\subsubsection*{Deep fully connected network}

Figure \ref{fig:program-graph-deep} and the corresponding $\netsor$ program \ref{alg:deep} extend the shallow architecture to any $M$-layer deep feed-forward network. The $\sH$-vars propagate through a sequence of matrix multiplications parameterized by independent $\sA$-vars $\sW^{(\ell)}$ and componentwise nonlinearities $\phi_\ell$. In this case Theorem \ref{thm:main} provides joint convergence of the pre-activations to a centered Gaussian process:
$\mathcal{G} = (\mathcal{G}^{(\ell)})_{\ell=1,\ldots,M} \sim \mathcal N \bra{0,K^{(M)}\Id_{d_{\rm out}}}$,
where $\Id_n$ denotes the identity matrix in $\R^{n\times n}$, recursively defining the deep NNGP kernel:
\[
K^{(1)} = \frac{c^{(0)}}{d}x^\top x, \quad K^{(\ell+1)} = \frac{c^{(\ell)}}{n_\ell} \E\left[\phi^{(\ell)}(\mathcal G^{(\ell)})^\top\phi^{(\ell)}(\mathcal G^{(\ell)})\right]\in \R,
\]
where $W^{(\ell)}_{ij}\sim\mathcal N\bra{0,c^{(\ell)}/n_\ell}$, for each $i,j \le n_\ell$ and $\ell\le M-1$.
Note that the NNGP kernel may also be written as a function of the $\sH$-vars of the program via the infinite-width product $\ldot$ defined in Equation \eqref{eq:infinite-width-product}:
\[K^{(\ell+1)} = c^{(\ell)}\bar {h}^{(\ell)}\ldot\bar {h}^{(\ell)}, \quad \text{for each }\ell \le M.\]

Note that, in our notation, $\mathcal G^{(\ell)} = \bar g^{(\ell)} / \sqrt{n_\ell}$ for each $\ell$, so the definition of the NNGP at each layer is consistent with the shallow network example.
Restricting the bound in Theorem \ref{thm:main} to the $\sG$-vars of the program recovers the bound from the main result in \cite{BasteriTrevisan2023}:
\begin{equation}
\W_p\left(\left(\frac{g^{(\ell)}}{\sqrt{n_\ell}}\right)_{\ell=1,\ldots,M},\, \mathcal{G}\right) \le\scc
\sum_{i=1}^{M-1}\frac1{\sqrt n_i}.
\end{equation}
Note that for the fully connected architecture this rate has been improved in \cite{trevisan2023wide} to $\mathcal O \bra{\bra{\min_i n_i}^{-1}}$.

\begin{algorithm}[htpb]
\floatname{algorithm}{Netsor program}
\caption{Deep feed-forward network}
\label{alg:deep}
\begin{algorithmic}
\Require $\sx: \sH(d)$
\Require $\sW^{(0)}: \sA(n_1, d)$
\Require $\sW^{(\ell)}: \sA(n_{{\ell+1}}, n_\ell)$ for $\ell = 1, \dots, M-2$
\Require $\sW^{(M-1)}: \sA(d_{\rm{out}}, n_{M-1})$
\State \algline $\sh^{(0)} = \sx: \sH(d)$ 
\For{$\ell =1, \ldots, M$}
  \State \algline $\sg^{(\ell)} = \sW^{(\ell-1)} \sh^{(\ell-1)}: \sG(n_\ell)$ \Comment{\matmul}
  \State \algline $\sh^{(\ell)} = \phi( \sg^{(\ell)}): \sH(n_\ell)$ \Comment{\nonlin}
\EndFor
\end{algorithmic}
\end{algorithm}

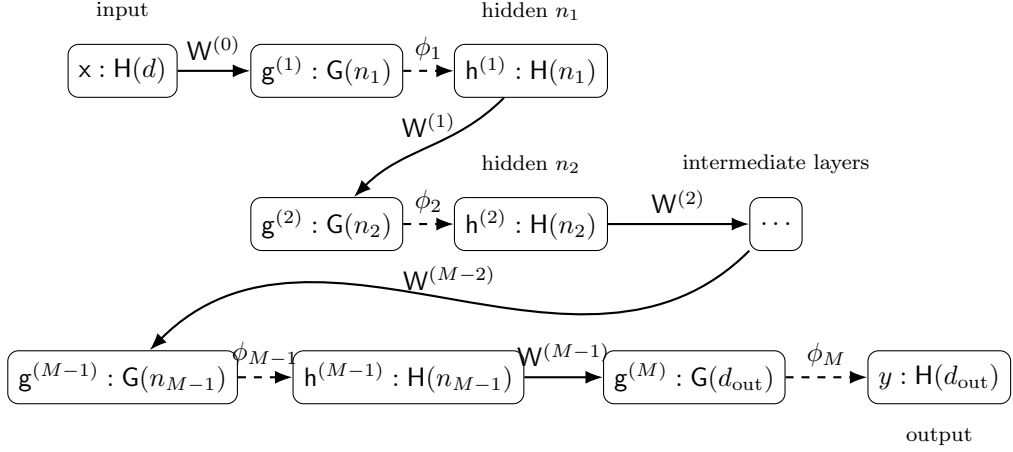
\begin{figure}[h]
\centering
\begin{tikzpicture}[
  font=\small,
  x=2.7cm,
  y=1.35cm,
  var/.style={rectangle, draw, rounded corners, minimum height=7mm, inner sep=4pt, align=center},
  edge/.style={->, thick, >=Latex},
  nonlin/.style={->, thick, dashed, >=Latex}
]
\node[var] (x)  at (0,0) {$\sx:\sH(d)$};
\node[var] (g1) at (1,0) {$\sg^{(1)}:\sG(n_1)$};
\node[var] (h1) at (2,0) {$\sh^{(1)}:\sH(n_1)$};

\node[var] (g2) at (1,-1.5) {$\sg^{(2)}:\sG(n_2)$};
\node[var] (h2) at (2,-1.5) {$\sh^{(2)}:\sH(n_2)$};
\node[var] (dots) at (3.2,-1.5) {$\ldots$};

\node[var] (g3) at (0,-3) {$\sg^{(M-1)}:\sG(n_{M-1})$};
\node[var] (h3) at (1.4,-3) {$\sh^{(M-1)}:\sH(n_{M-1})$};
\node[var] (gL) at (2.8,-3) {$\sg^{(M)}:\sG(d_{\rm out})$};
\node[var] (y) at (4,-3) {$y:\sH(d_{\rm out})$};

\draw[edge] (x) -- node[above] {$\sW^{(0)}$} (g1);
\draw[nonlin] (g1) -- node[above] {$\phi_1$} (h1);

\draw[edge] (h1) to[out=225,in=45] node[above] {$\sW^{(1)}$} (g2);
\draw[nonlin] (g2) -- node[above] {$\phi_2$} (h2);
\draw[edge] (h2) -- node[above] {$\sW^{(2)}$} (dots);
\draw[edge] (dots) to[out=225,in=45] node[above] {$\sW^{(M-2)}$} (g3);
\draw[nonlin] (g3) -- node[above] {$\phi_{M-1}$} (h3);

\draw[edge] (h3) -- node[above] {$\sW^{(M-1)}$} (gL);
\draw[nonlin] (gL) -- node[above] {$\phi_{M}$} (y);

\node[above=0.18cm of x] {\scriptsize input};
\node[above=0.18cm of h1] {\scriptsize hidden $n_1$};
\node[above=0.18cm of h2] {\scriptsize hidden $n_2$};
\node[above=0.18cm of dots] {\scriptsize intermediate layers};
\node[below=0.18cm of y] {\scriptsize output};
\end{tikzpicture}
\caption{Program graph of a deep feed-forward network with activated output. The internal widths are
$n_1,\ldots, n_{M-1}$. The terminal output dimension $d_{\rm out}$ is structural, even if the final
program line is obtained by a componentwise nonlinearity.}
\label{fig:program-graph-deep}
\end{figure}

\subsubsection*{Time-unrolled RNN with residue}

The graph in Figure \ref{fig:program-graph-rnn} and the corresponding $\netsor$ program \ref{alg:rnn} illustrate a recurrent neural network unrolled over two time steps. The structurally defining feature of this architecture is the reuse of the $\sA$-vars $\sW^x$ and $\sW^h$ across the temporal sequence, which introduces dependencies between the generated $\sH$-vars, as opposed to the two previous examples.
In particular, the results for fully connected architectures are no longer applicable in this case.
Under the $\netsor$ formalism, Theorem \ref{thm:main} guarantees that the unrolled hidden states jointly converge to the recurrent NNGP \cite{yang1} with the expected rate. Specifically, we have
\begin{equation}
\W_p\left(\left(\frac{h_1}{\sqrt{n}}, \frac{h_2}{\sqrt{n}},\frac{y}{\sqrt{d_{\rm{out}}}}\right),\, (Z_1, Z_2, \mathcal G)\right) \le \frac{\scc}{\sqrt{n}},
\end{equation}
where \(\mathcal G=\bar y/\sqrt{d_{\rm out}}\) denotes the recurrent NNGP associated with this architecture and $Z_1 = \phi(U_1), \; Z_2 = \psi(V,U_2)$ with
 \[U_1 \sim \mathcal N \bra{0, \frac 1 d x_1^\top x_1}, \quad U_2 \sim \mathcal N\bra{0,\frac{1}{d}x_2^\top x_2}, \quad \text{and } V \sim \mathcal N\bra{0,\frac{1}{n}\E_Z\left[\phi(U_1)^\top \phi(U_1)\right]}.\]
Then the NNGP \(\mathcal G=\bar y/\sqrt{d_{\rm out}}\) is, explicitly, a $\R^{d_{\rm out}}$-valued random variable with i.i.d.~components centered Gaussian with variance:
\[
\Sigma_{\mathcal G} = \Id_{d_{\rm out}}\E\left[\psi(V,U_2)^\top \psi(V,U_2)\right].
\]
Note that this agrees with the choice of covariance in the infinite-width execution \eqref{eq:sigma-gg-mf}.

\begin{algorithm}[htpb]
\floatname{algorithm}{Netsor program}
\caption{Time-unrolled recurrent architecture}
\label{alg:rnn}
\begin{algorithmic}
\Require $\sx_1, \sx_2: \sH(d)$
\Require $\sW^x: \sA(n, d)$
\Require $\sW^h: \sA(n, n)$
\Require $\sW^o: \sA(d_{\rm out}, n)$
\State \algline $\sg_1 = \sW^x \sx_1: \sG(n)$ \Comment{\matmul}
\State \algline $\sh_1 = \phi(\sg_1): \sH(n)$ \Comment{\nonlin}
\State \algline $\sg_3 = \sW^h \sh_1: \sG(n)$ \Comment{\matmul}
\State \algline $\sg_2 = \sW^x \sx_2: \sG(n)$ \Comment{\matmul}
\State \algline $\sh_2 = \psi(\sg_2, \sg_3): \sH(n)$ \Comment{\nonlin}
\State \algline $\sy = \sW^o \sh_2: \sG(d_{\rm out})$ \Comment{\matmul}
\end{algorithmic}
\end{algorithm}

\begin{figure}[h]
\centering
\begin{tikzpicture}[
  font=\small,
  x=2.8cm,
  y=1.35cm,
  var/.style={rectangle, draw, rounded corners, minimum height=7mm, inner sep=4pt, align=center},
  edge/.style={->, thick, >=Latex},
  nonlin/.style={->, thick, dashed, >=Latex}
]
\node[var] (x1) at (0,0) {$\sx_1:\sH(d)$};
\node[var] (g1) at (1,0) {$\sg_1:\sG(n)$};
\node[var] (h1) at (2,0) {$\sh_1:\sH(n)$};
\node[var] (g3) at (3,0) {$\sg_3:\sG(n)$};

\node[var] (x2) at (0,-1.5) {$\sx_2:\sH(d)$};
\node[var] (g2) at (1,-1.5) {$\sg_2:\sG(n)$};
\node[var] (h2) at (2,-1.5) {$\sh_2:\sH(n)$};

\node[var] (y) at (1,-3.0) {$\sy:\sG(d_{\rm out})$};

\draw[edge] (x1) -- node[above] {$\sW^x$} (g1);
\draw[nonlin] (g1) -- node[above] {$\phi$} (h1);
\draw[edge] (h1) -- node[above] {$\sW^h$} (g3);

\draw[edge] (x2) -- node[above] {$\sW^x$} (g2);
\draw[nonlin] (g3) -- node[right] {$\psi$} (h2);
\draw[nonlin] (g2) -- node[above] {$\psi$} (h2);

\draw[edge] (h2) -- node[right] {$\sW^o$} (y);

\node[left=0.18cm of x1] {\scriptsize input time 1};
\node[left=0.18cm of x2] {\scriptsize input time 2};
\node[below=0.18cm of y] {\scriptsize output};
\end{tikzpicture}
\caption{Time-unrolled recurrent architecture. The same input weight $\sW^x$ is reused
across time, and $\sW^h$ maps the previous hidden state $\sh_1$ into the next preactivation $\sg_3$.
The postactivation is defined as $\sh_2 = \psi(\sg_2,\sg_3)$.
The recurrent dimension $n$ is an internal hidden width.}
\label{fig:program-graph-rnn}
\end{figure}
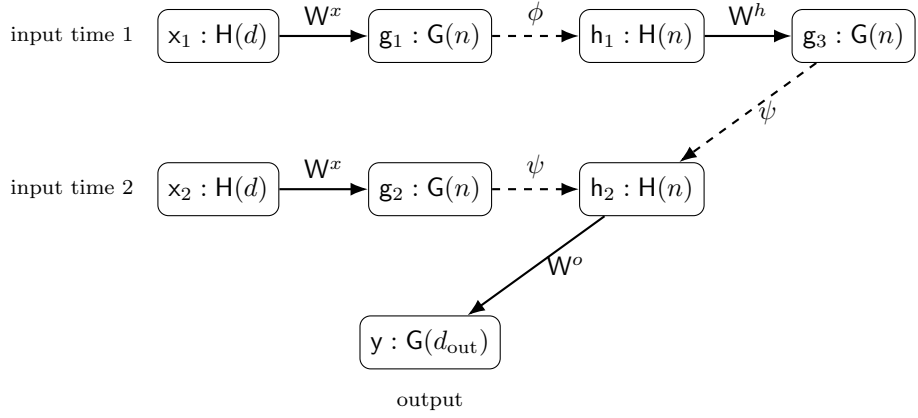

\subsubsection*{Two-input shallow network}

To analyze the covariance structure of the outputs given multiple data points, Figure \ref{fig:program-graph-two-inputs} depicts a shallow network evaluated on two distinct inputs $\sx^1$ and $\sx^2$. 
The corresponding $\netsor$ program \ref{alg:two_inputs} is obtained by duplicating the lines of the shallow network program \ref{alg:shallow}.
Theorem \ref{thm:main} yields the expected estimate:
\begin{equation}
\W_p \left(\left(\frac{y^1}{\sqrt{d_{\rm{out}}}}, \frac{y^2}{\sqrt{d_{\rm{out}}}}\right), \overline{\mathcal G}\right)\le \frac{\scc}{\sqrt{n}},
\end{equation}
where $\overline{\mathcal G}= \bra{\mathcal G^1,\mathcal G^2}$ with $\mathcal G^j$ the NNGP  for the shallow architecture applied to the inputs $x^j$, for $j=1,2$.

\begin{algorithm}
\floatname{algorithm}{$\netsor$ program}
\caption{Two-input shallow network}
\label{alg:two_inputs}
\begin{algorithmic}
\Require $\sx^1, \sx^2: \sH(d)$
\Require $\sW^{(0)}: \sA(n, d)$
\Require $\sW^{(1)}: \sA(d_{\rm{out}},n)$
\State \algline $\sg^1 = \sW^{(0)} \sx^1: \sG(n)$ \Comment{\matmul}
\State \algline $\sg^2 = \sW^{(0)} \sx^2: \sG(n)$ \Comment{\matmul}
\State \algline $\sh^1 = \phi( \sg^1): \sH(n)$ \Comment{\nonlin}
\State \algline $\sh^2 = \phi( \sg^2): \sH(n)$ \Comment{\nonlin}
\State \algline $\sy^1 = \sW^{(1)} \sh^1: \sG(d_{\rm{out}})$ \Comment{\matmul}
\State \algline $\sy^2 = \sW^{(1)} \sh^2: \sG(d_{\rm{out}})$ \Comment{\matmul}
\end{algorithmic}
\end{algorithm}

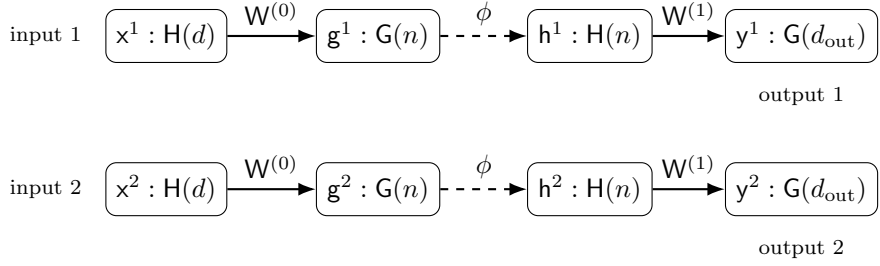
\begin{figure}[h]
\centering
\begin{tikzpicture}[
  font=\small,
  x=2.8cm,
  y=1.35cm,
  var/.style={rectangle, draw, rounded corners, minimum height=7mm, inner sep=4pt, align=center},
  edge/.style={->, thick, >=Latex},
  nonlin/.style={->, thick, dashed, >=Latex}
]
\node[var] (x1) at (0,0) {$\sx^1:\sH(d)$};
\node[var] (g1) at (1,0) {$\sg^1:\sG(n)$};
\node[var] (h1) at (2,0) {$\sh^1:\sH(n)$};

\node[var] (x2) at (0,-1.5) {$\sx^2:\sH(d)$};
\node[var] (g2) at (1,-1.5) {$\sg^2:\sG(n)$};
\node[var] (h2) at (2,-1.5) {$\sh^2:\sH(n)$};

\node[var] (y1) at (3,0) {$\sy^1:\sG(d_{\rm out})$};
\node[var] (y2) at (3,-1.5) {$\sy^2:\sG(d_{\rm out})$};

\draw[edge] (x1) -- node[above] {$\sW^{(0)}$} (g1);
\draw[edge] (x2) -- node[above] {$\sW^{(0)}$} (g2);

\draw[nonlin] (g1) -- node[above] {$\phi$} (h1);
\draw[nonlin] (g2) -- node[above] {$\phi$} (h2);

\draw[edge] (h1) -- node[above] {$\sW^{(1)}$} (y1);
\draw[edge] (h2) -- node[above] {$\sW^{(1)}$} (y2);

\node[left=0.18cm of x1] {\scriptsize input 1};
\node[left=0.18cm of x2] {\scriptsize input 2};
\node[below=0.18cm of y1] {\scriptsize output 1};
\node[below=0.18cm of y2] {\scriptsize output 2};
\end{tikzpicture}
\caption{A shallow network evaluated on two inputs. The same $\sA$-vars are reused for the two
inputs, represented by repeated edge labels $\sW^{(0)}$ and $\sW^{(1)}$. The width $n$
is hidden and is common to both branches.}
\label{fig:program-graph-two-inputs}
\end{figure}

\subsection{Extensions}
\label{subsec:extensions-kernel}

The $\netsorpl$ language introduced in \cite{yang1} allows rather general scalar-valued operations ($\moment$) yielding additional parameters that can be used in functions of $\nonlin$ lines. This extension covers in particular attention layers, where scalar variables are used mainly to store empirical kernels, and these kernels are then fed into scalar nonlinearities such as the softmax. In this section, we describe a similar extension, denoted here by $\netsork$, that can be obtained with minimal modification of our arguments, but is still sufficient to cover attention and normalization of layers.
Our main result Theorem \ref{thm:main} extends naturally to this extended language, effectively covering attention-based architectures.

\subsubsection*{$\netsork$ programs}

A $\netsork$ program is a $\netsor$ program, augmented with $\sC$-variables: in the finite-width execution they are real-valued random variables, while in the infinite-width execution they are deterministic constants. In addition to the usual $\matmul$ and $\nonlin$ instructions, we allow the following operations:

\begin{itemize}
\item \emph{Kernel variables.} If $\sh,\sh':\sH(n)$, then
\[
(\kernel)\qquad
\ss:=\sh^\star\sh':\sC .
\]
\item \emph{Nonlinearities with scalar parameters.} If $\phi:\R^k\times\R^u\to\R$, if $\sh_1,\ldots,\sh_k:\sH(n)$ and $\ss_1,\ldots,\ss_u:\sC$, then
\[
(\nonlinpl)\qquad
\sh:=\phi(\sh_1,\ldots,\sh_k;\ss_1,\ldots,\ss_u):\sH(n),
\]
where $\phi$ is to be interpreted as acting componentwise on the $\sH$-vars.
\item \emph{Scalar maps.} If $\psi:\R^u\to \R$ and $\ss_1,\ldots,\ss_u:\sC$, then
\[
(\scal)\qquad
\ss:=\psi(\ss_1,\ldots,\ss_u):\sC .
\]
\end{itemize}

\subsubsection*{Extended program execution}

The execution rules for $\netsork$ naturally extend those of $\netsor$. Keeping $\matmul$ and $\nonlin$ unchanged, the execution of a $\kernel$ line $\ss=\sh^\star\sh'$ is defined in finite-width interpretation as $s=\frac1n h^\top h'$, and in infinite-width as $\bar s=\bar h\ldot \bar h'$. The $\scal$ and $\nonlinpl$ instructions are executed as componentwise compositions. Moreover, we assume that input $\sC$-variables, if present, are to be executed as deterministic constants. 

Hence, every $\sC$-var generated by a $\netsork$ program is a measurable function of the previously generated $\sH$-vars. In particular, scalar variables do not introduce additional observations of the weights beyond those already encoded by the filtration. Consequently, the conditional Gaussian formulas for $\matmul$ lines remain identical to those of the base $\netsor$ language.

We impose the following regularity assumption on $\netsork$ programs: functions $\phi$ appearing in $\nonlinpl$ lines satisfy local Lipschitz estimate with linear growth. Namely, there exists a constant $C<\infty$ such that, for all $x,y\in\R^k$ and $s,t\in\R^u$,
\begin{equation}
\label{eq:separated-linear-lip}
|\phi(x;s)-\phi(y;t)|
\le
C(1+\nor{s}+\nor{t})\,\nor{x-y}
+
C(1+\nor{x}+\nor{y})\,\nor{s-t}.
\end{equation}

\begin{remark}[Relation with Yang's $\moment$ instruction]
\label{rmk:netsork-vs-moment}
The distinction between $\netsork$ and the $\netsorpl$ language of \cite{yang2} is mainly analytic rather than algebraic. In Yang's formulation, scalar variables are generated by a general $\moment$ instruction of the form
\[
\textsc{(Moment)} \quad 
\ss
:=
\frac1n\sum_{i=1}^n
\psi\bigl(\sg^1_i,\ldots,\sg^k_i;\ss_1,\ldots,\ss_u\bigr)
:\sC .
\]
Whenever the componentwise quantity can itself be represented as the $i$-th coordinate of an $\sH(n)$-var, the corresponding moment can be written as a kernel against the constant-one vector, $\mathbf 1_n:=(1,\ldots,1)\in \R^n$. Thus, at a purely formal level, $\kernel$ together with scalar operations can reproduce any $\moment$ instructions.

The reason for isolating the restricted language $\netsork$ is that the quantitative argument requires stability estimates compatible with the norms controlled by the main theorem. Kernel variables are controlled directly by Corollary \ref{coro:kernel-lln}, while subsequent nonlinear maps are required to satisfy the Lipschitz condition \eqref{eq:separated-linear-lip}. By contrast, a fully general $\moment$ instruction with a locally Lipschitz integrand of polynomial growth would require additional empirical high-moment estimates. Hence, the restriction to $\kernel$ variables serves to state a quantitative theorem under transparent and checkable analytic assumptions.
\end{remark}

\subsubsection*{Extended program graphs}
\label{subsec:program-graphs-netsork}

For $\netsork$ programs, the directed program graph must also record scalar variables. In addition to $\sH$- and $\sG$-nodes, we include nodes of type $\sC$. A $\kernel$ line $\ss=\sh^\star\sh'$ contributes directed edges $\sh\longrightarrow \ss$, $\sh'\longrightarrow \ss$, which we draw as dotted arrows with the label $\star$. A $\scal$ line $\ss=\psi(\ss_1,\ldots,\ss_u)$ contributes directed edges $\ss_i\to \ss$, labelled by $\psi$. Finally, a scalar-parametric function contributes edges both from the $\sH$-variables and from the scalar variables to the output $\sh$.

The definition of a hidden width expands accordingly: a dimension $n$ is a hidden width if there exists an $\sH$-var of shape $n$ produced by a $\matmul$ line whose downstream descendants are later contracted as the parent of a further $\matmul$ line, or within a $\kernel$ line $\ss=\sh^\star\sh'$. We continue to denote by $\nhidd$ the collection of such widths. This convention is important for attention layers. The query and key vectors, denoted in the program by $\sq$ and $\sk$ respectively, have dimension $u \in \nhidd$, and the pre-softmax scores
\[\sw^{ij}=\sq^i{}^\star\sk^j, \quad \text{for each } i,j \le u,\]
are obtained by contracting over that dimension. Hence $u$ is a hidden width even if the corresponding variables are not later used as parents of another $\matmul$ line.

Dimensions not
belonging to $\nhidd$ are either considered input $\nin$, hence structural, or output $\nout$.

\begin{example}[Single-head attention]
\label{ex:attention-netsork}

The following program describes a single layer of single-head attention with deterministic kernel scaling. The nonlinear functions involved are 
\[
\layernorm_\varepsilon(x;\mu,\sigma):=\frac{x-\mu}{\sqrt{\sigma^2 + \varepsilon }},
\qquad
\softmax(z_1,\ldots,z_u)_j
:=
\frac{\exp(z_j)}{\sum_{m=1}^u\exp(z_m)}.
\]

Condition \eqref{eq:separated-linear-lip} is tailored to the attention output operation $(v^1,\ldots,v^u; a^{i1},\ldots,a^{iu}) \longmapsto \sum_{j=1}^u a^{ij}v^j$, which is not globally Lipschitz jointly in $(v,a)$, but satisfies the locally Lipschitz bound with $C=1$. For layer normalization $\layernorm_\varepsilon$, the nonlinearity also satisfies \eqref{eq:separated-linear-lip} with $C$ depending on $\varepsilon$.

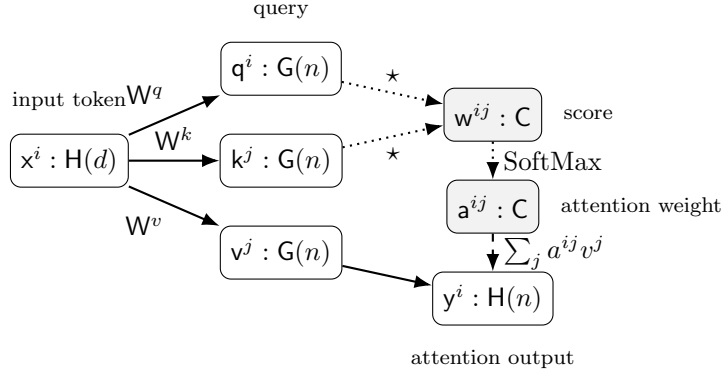
\begin{figure}[htpb]
\centering
\begin{tikzpicture}[
  font=\small,
  x=2.8cm,
  y=1.35cm,
  var/.style={rectangle, draw, rounded corners, minimum height=7mm, inner sep=4pt, align=center},
  cvar/.style={rectangle, draw, rounded corners, minimum height=7mm, inner sep=4pt, align=center, fill=gray!10},
  edge/.style={->, thick, >=Latex},
  nonlin/.style={->, thick, dashed, >=Latex},
  scalar/.style={->, thick, dotted, >=Latex}
]
\node[var] (x)  at (0,0) {$\sx^i:\sH(d)$};
\node[var] (q)  at (1,0.9) {$\sq^i:\sG(n)$};
\node[var] (k)  at (1,0) {$\sk^j:\sG(n)$};
\node[var] (v)  at (1,-0.9) {$\sv^j:\sG(n)$};

\node[cvar] (w) at (2,0.45) {$\sw^{ij}:\sC$};
\node[cvar] (a) at (2,-0.45) {$\sa^{ij}:\sC$};

\node[var] (y)  at (2,-1.35) {$\sy^i:\sH(n)$};

\draw[edge] (x) -- node[above left] {$\sW^q$} (q);
\draw[edge] (x) -- node[above] {$\sW^k$} (k);
\draw[edge] (x) -- node[below left] {$\sW^v$} (v);

\draw[scalar] (q) -- node[above] {$\star$} (w);
\draw[scalar] (k) -- node[below] {$\star$} (w);

\draw[scalar] (w) -- node[right] {$\softmax$} (a);

\draw[nonlin] (a) -- node[right] {$\sum_j a^{ij}v^j$} (y);
\draw[edge] (v) -- (y);

\node[above=0.18cm of x] {\scriptsize input token};
\node[above=0.18cm of q] {\scriptsize query};
\node[right=0.15cm of w] {\scriptsize score};
\node[right=0.15cm of a] {\scriptsize attention weight};
\node[below=0.18cm of y] {\scriptsize attention output};
\end{tikzpicture}
\caption{Schematic program graph for a single-head attention layer in $\netsork$. Solid arrows are $\matmul$ operations labelled by $\sA$-vars, dotted arrows indicate kernel/scalar operations, and dashed arrows indicate scalar-parametric nonlinearities. The width $n$ is produced by the projections and contracted in the kernel scores $\sw^{ij}=\sq^i {}^\star\sk^j$.}
\label{fig:program-graph-attention}
\end{figure}

\begin{algorithm}[htpb]
\floatname{algorithm}{Netsor program}
\caption{A single attention layer in $\netsork$}
\label{alg:attention-netsork}
\begin{algorithmic}
\Require $\{\sx^i\}_{i=1}^u:\sH(d)$
\Require $\{\mu^i,\sigma^i\}_{i=1}^u:\sC$ \Comment{input LayerNorm statistics, with $\sigma^i>0$}
\Require $\sW^q,\sW^k,\sW^v:\sA(n,d)$

\Statex \Comment{1. Layer normalization}
\For{$i=1,\ldots,u$}
  \State \algline $\tilde{\sx}^i:=\layernorm_\varepsilon(\sx^i;\mu^i,\sigma^i):\sH(d)$ \Comment{$\nonlinpl$}
\EndFor

\Statex \Comment{2. Projections}
\For{$i=1,\ldots,u$}
  \State \algline $\sq^i:=\sW^q\tilde{\sx}^i:\sG(n)$ \Comment{$\matmul$}
  \State \algline $\sk^i:=\sW^k\tilde{\sx}^i:\sG(n)$ \Comment{$\matmul$}
  \State \algline $\sv^i:=\sW^v\tilde{\sx}^i:\sG(n)$ \Comment{$\matmul$}
\EndFor

\Statex \Comment{3. Pre-softmax scores}
\For{$i,j=1,\ldots,u$}
  \State \algline $\sw^{ij}:=\sq^i {}^\star\sk^j:\sC$ \Comment{$\kernel$}
\EndFor

\Statex \Comment{4. Softmax}
\For{$i=1,\ldots,u$}
  \For{$j=1,\ldots,u$}
    \State \algline $\sa^{ij}:=\softmax(\sw^{i1},\ldots,\sw^{iu})_j:\sC$ \Comment{\scal}
  \EndFor
\EndFor

\Statex \Comment{5. Attention output}
\For{$i=1,\ldots,u$}
  \State \algline $\sy^i:=\sum_{j=1}^u \sa^{ij}\sv^j:\sH(n)$ \Comment{$\nonlinpl$}
\EndFor
\end{algorithmic}
\end{algorithm}

For readability, Figure~\ref{fig:program-graph-attention} displays only the dependencies associated with a fixed query token $i$ and key/value token $j$; the full attention layer is obtained by repeating the same operation for all $i,j\le u$. Note that query and key vectors have dimension $n$, and the pre-softmax score is obtained by contracting over that dimension. Hence, $n$ is a hidden width even if the variables are not later used as parents of another $\matmul$ instruction.

Successive iterations of the attention block in Algorithm \ref{alg:attention-netsork} make $\layernorm$ act on non-input $\sH$-vars. In this case the statistics $\mu$ and $\sigma$ can be computed via the new instructions with:
\begin{alignat*}{3}
    &(\kernel)   &\qquad \mu      &= \sh^\star \mathbf 1_n &&\qquad\colon \sC,\\
    &(\nonlinpl) &\qquad \sh_c    &= \sh - \mathbf \mu     &&\qquad\colon \sH(n),\\
    &(\kernel)   &\qquad \sigma^2 &= \sh_c^\star\sh_c      &&\qquad\colon \sC,\\
    &(\scal)     &\qquad \sigma   &= \sqrt{\sigma^2}       &&\qquad\colon \sC.
\end{alignat*}
\end{example}

\begin{remark}
With the $1/n$ normalization in the $\kernel$ rule, the attention scores have deterministic infinite-width limits. This is the law-of-large-numbers scaling. A different scaling, for instance the fluctuation scaling $1/\sqrt{n}$ for centered kernels as the one used in \cite{sakai2026infinitewidthlimitsingleattention}, would lead to Gaussian limits for the scalar variables and requires a quantitative CLT rather than the kernel LLN used here.
\end{remark}

\subsubsection*{Extended convergence result}

With the above formalization, we state the extension of Theorem~\ref{thm:main} to $\netsork$ programs,
under an explicit non-degeneracy assumption for $\sG$-vars in the infinite-width limit.


\begin{theorem}
\label{thm:main-netsork}
Consider a $\netsork$ program such that nonlinearities in $\nonlin$, $\scal$ are Lipschitz and those in $\nonlinpl$ satisfy the local Lipschitz bound \eqref{eq:separated-linear-lip}. Let
\[
\big(h^{(\ell)}\big)_{\ell=1,\ldots ,H}, \big(s^{(j)}\big)_{j=1,\ldots ,S}, \qquad \big(\bar h^{(\ell)}\big)_{\ell=1,\ldots ,H}, \big(\bar s^{(j)}\big)_{j=1,\ldots ,S},
\]
be the finite-width and infinite-width executions of its $\sH$-vars and $\sC$-vars,
where $H$ and $S$ denote, respectively, the number of $\sH$-vars and $\sC$-vars defined in the program,
and assume that its $\sG$-vars are jointly non-degenerate Gaussians in the infinite-width execution. Then, for every $p\ge 1$, there exists a constant $\scc<\infty$, depending only on $p$ and on the structural constants of the program, such that
\begin{align}
  \label{eq:main-netsork}
  \begin{split}
\W_p&\left( \left( \left( \frac{ h^{(\ell)} } {\sqrt{n_\ell} }  \right)_{\ell=1,\ldots ,H}, \left(s^{(j)} \right)_{j=1,\ldots ,S} \right), \left( \left( \frac{ \bar h^{(\ell) }}{\sqrt{n_\ell}} \right)_{\ell=1,\ldots ,H}, \big(\bar s^{(j)}\big)_{j=1,\ldots ,S}\right) \right)
\\
&\quad \le
\scc\sum_{m \in \nhidd } \frac1{\sqrt{m}}.
  \end{split}
\end{align}
\end{theorem}

\begin{remark}
In the basic $\netsor$ setting we used an equivalence reduction to remove degenerate
$\matmul$ lines. For $\netsork$ this reduction is more delicate:
a scalar kernel variable is deterministic in infinite-width but random at finite width, so a
linear relation that holds identically in the infinite-width execution need not hold identically
in the finite-width execution. Specifically, the pointwise lifting of infinite-width linear identities to finite-width linear identities utilized in the proof of Proposition \ref{prop:matmul-reduction} fails when these identities depend on the fluctuations of empirical scalar kernels. For this reason we state the extension under an explicit
non-degeneracy assumption for $\sG$-vars in the infinite-width execution.
\end{remark}

The proof of this theorem requires an extension of our inductive line-by-line coupling arguments to encompass scalar variable executions and is given in Subsection~\ref{sec:proof-netsork}.


\section{Technical lemmas}\label{sec:tech}

In this section we collect some technical results that are used in the proof of Theorem \ref{thm:main}.
We begin by setting some notation:
$\|\cdot\|,\|\cdot\|_{\op}$ and $\|\cdot\|_{L^p}$ denote the Euclidean norm, operator norm and $L^p$ norm, respectively.
Unless specified, the norm inner to the $L^p$ norm is the Frobenius norm.
As anticipated in the introduction, $\scc$ denotes a positive constant that may change from line to line.
Let $S,T$ in $\N$. 
For any block matrix
  \[
    M=\begin{pmatrix} A & B \\ C & D \end{pmatrix} \in \R^{(S+T)\times (S+T)},
  \]
we define
its generalized Schur complement
  \begin{equation}
    \label{eq:schur}
        M/D := A-BD^+C,
  \end{equation}
  where the superscript $+$ denotes the Moore-Penrose pseudoinverse
  \begin{equation}
  \label{eq:moore-penrose}
  M^+ = \lim_{\delta \downarrow 0} (M^\top M + \delta \Id_T)^{-1} M^\top,
  \end{equation}
for any $M \in \R^{S\times T}$.
If the block matrix $M$ as above is symmetric positive semidefinite, then so is $M/D$. 
This extension matches the classical Schur complement when $D$ is invertible, and is the relevant form for Gaussian conditioning with singular covariance matrices, see Equation \eqref{eq:conditional-gaussian} below.

We also recall Greville's formula for the pseudoinverse of a rank-one update of a matrix \citep{Udwadia1997, greville}.
Let $B \in \R^{S\times T}$ be partitioned as $B = [A,a]$, where $A$ consists of the first $T-1$ columns of $B$ and $a$ is the last column.
Then the pseudoinverse of $B$ is given by
\begin{equation}
  \label{eq:greville}
  B^+ = \begin{pmatrix} A^+ - A^+ a b \\ b \end{pmatrix},
\end{equation}
where 
\begin{equation}
  \label{eq:greville-b}
  b = \begin{cases} (a - A A^+a)^+ & \text{if } (\Id_S - A A^+) a \neq 0, \\ 
     (1+a^\top( A^\top A)^+ a)^{-1} a^\top( A^\top A)^+ & \text{otherwise.}
  \end{cases}
\end{equation}
This formula adapts the well known Woodbury-Sherman-Morrison \cite{sherman-morrison} formula to rectangular  matrices and their pseudoinverses.

Furthermore, the generalized Schur complement, the pseudoinverse, and Greville's formula can be understood from an operator-theoretic point of view, see e.g.\ \cite{Ando2005}. Whenever $B$ represents a linear operator between two finite-dimensional Hilbert spaces $\mathcal H$ and $\mathcal H'$,
its pseudoinverse coincides with the matrix pseudoinverse as long as orthonormal bases of $\mathcal H$ and $\mathcal H'$ are chosen to represent the operator,
where $A^\top$ in \eqref{eq:moore-penrose} becomes the adjoint operator to $A$, defined as 
\[
\langle Ax, y \rangle_{\mathcal H'} = \langle x, A^\top y\rangle_{\mathcal H}, \quad \text{for all }x \in \mathcal H \text{ and } y \in \mathcal H',
\]
whenever $A$ is a linear operator from $\mathcal H$ to $\mathcal H'$.
Note that the pseudoinverse explicitly depends on the inner products of $\mathcal H$ and $\mathcal H'$.
Under these considerations, Greville's formula \eqref{eq:greville} can be applied to finite-dimensional Hilbert spaces.

\subsection{Random variables and moment bounds}
\label{subsec:mf-moment-bounds}

Given $p\ge 1$ and a probability space $(\Omega, \cA, \PP)$, we write $X \in L^p(\Omega, \cA, \PP; \R^S)$ to denote a random variable $X: \Omega \to \R^{S}$ such that 
\[\nor{X}_{L^p} := \EE\sqa{ \nor{X}^p}^{1/p} < \infty.\]
Since we are mostly interested in convergence in law, we often do not specify the underlying probability space for random variables, and also write for brevity $X \in \R^S$ instead of $X: \Omega \to \R^S$,  $X \stackrel{law}{=} Y$ to denote that two random variables $X$ and $Y$ have the same law.

The following elementary concentration result will be paramount for proving Theorem \ref{thm:main}.

\begin{lemma}\label{lem:concentration}
 Let $p \ge 1$, let $X$ be a random variable with values in $\R^S$ and let $f: \R^S \to \R^T$ be differentiable at $\bar{x} \in \R^S$. 
 Then, for some constant $\cost = \cost(f, \bar{x})<\infty$ it holds
 \begin{equation}
  \nor{ f(X) - f(\bar{x} )}_{L^p} \le \cost \bra{ \nor{X - \bar x}_{L^p} + \bra{  \nor{f(X)}_{L^{2p}}+ \nor{f(\bar{x})}} \nor{X- \bar x }_{L^{2p}} }.
 \end{equation}
\end{lemma}

\begin{proof}
By differentiability of $f$ at $\bar{x}$, there exists $\eps>0$ and $M < \infty $ such that, if $\nor{x-\bar{x}}\le \eps$, then 
\[\nor{f(x) - f(\bar{x})} \le M \nor{x-\bar{x}}.\]
Consider the event $A =: \cur{ \nor{X- \bar x } \le \eps}$, so that by Markov inequality:
\[ \nor{ \mathbbm{1}_{A^c} }_{L^{2p}} \le \eps^{-1} \nor{X- \bar x}_{L^{2p}}.\]
Then, 
  \begin{equation*} \begin{split} \nor{ f(X) - f(\bar{x}) }_{L^p} & \le \nor{ \bra{ f(X) - f(\bar{x})} \mathbbm{1}_A }_{L^p} + \nor{ \bra{ f(X) - f(\bar{x})} \mathbbm{1}_{A^c} }_{L^p}  \\
  & \le  M \nor{ X-\bar{x} }_{L^p} + \nor{ \bra{\nor{f(X)} + \nor{f(\bar{x})} } \mathbbm{1}_{A^c} }_{L^p}  \\
   & \le M \nor{ X-\bar{x} }_{L^p}  + \bra{ \nor{ f(X)}_{L^{2p}}  + \nor{ f(\bar{x})} } \nor{ \mathbbm{1}_{A^c}}_{L^{2p}}\\
      & \le \cost \bra{ \nor{ X-\bar{x} }_{L^p}  + \bra{ \nor{ f(X)}_{L^{2p}}  + \nor{ f(\bar{x})} }  \nor{ X-\bar{x} }_{L^{2p} } }
  \end{split} \end{equation*}
 with $\cost := \max\cur{M, \eps^{-1}}$.
\end{proof}

When considering matrix valued random variables $W \in \R^{S \times T}$, we will say that the covariance $\Sigma_W$ is \emph{separable} if there exist $L\in \R^{S\times S}$, $R\in \R^{T \times T}$ such that
\[ \Sigma_W = L \otimes R. \]

We will also rely on Rosenthal's inequality to bound the moments of sums of independent, centered random variables \cite{Rosenthal1970}.
Let $p \ge 2$, and let $X_1, \ldots, X_n \in L^p(\Omega; \R^S)$ be such random variables.
Then, there exists a constant $\scc_p$ depending only on $p$ such that
 \begin{equation}
  \label{eq:rosenthal}
\nor{ \sum_{i=1}^n X_i }_{L^p}^p \le \scc_p \max\cur{ \sum_{i=1}^n \nor{X_i}_{L^p}^p, \bra{ \sum_{i=1}^n \nor{X_i}_{L^2}^2 }^{p/2} }.
 \end{equation}

We now record a uniform moment bound for the infinite-width execution of $\netsor$ programs. 

\begin{lemma}[Uniform infinite-width moment bounds]
\label{lem:infinite-width-finite-moments}
Consider a $\netsor$ program such that all activations $\phi$ appearing in
$\nonlin$ lines have polynomial growth. Then, for every $p\ge 1$, there exists a
constant $\scc<\infty$, depending only on $p$ and on the structural constants of the
program, such that for every $\sH$-var $\sh:\sH(n)$ generated by the program, every
coordinate $i\le n$, and every infinite-width execution,
\[
\EE\big[|\bar h_i|^p\big]\le \scc .
\]
In particular, $\scc$ is independent of the hidden widths.
\end{lemma}

\begin{proof}
We prove the claim by induction over the program lines. Since the program is finite, it
is enough to show that each step propagates bounds depending only on the previous
bounds and on the structural data of the program, but not on the widths.

Input $\sH$-vars are deterministic. Hence, if $\sx:\sH(d)$ is an input, then
\[
|\bar x_i|^p\le \max_{x\in \X}\max_{1\le i\le d}|x_i|^p,
\]
which is finite and independent of the hidden widths.

Assume now that the bound has been proved for all $\sH$-vars available before a given
line.

If the line is of type $\matmul$, say $\sg=\sW\sh$ with $\sh:\sH(m)$ and
$\sg:\sG(n)$, then by the infinite-width execution rule each coordinate $\bar g_i$ is centered
Gaussian with variance
\[
\Var(\bar g_i)=\bar h\ldot \bar h
=
\frac1m\sum_{r=1}^m \EE[|\bar h_r|^2].
\]
By the induction hypothesis with exponent $p=2$,
\[
\Var(\bar g_i)\le \scc,
\]
uniformly in $m,n$ and in $i$. Therefore, by finiteness of Gaussian moments, there exists a positive constant $\scc_p$ with:
\[
\EE[|\bar g_i|^p]
\le \scc_p \big(\Var(\bar g_i)\big)^{p/2}
\le \scc
\]
This gives the desired uniform bound for the new $\sG$-var.

If the line is of type $\nonlin$, say
\[
\sh=\phi(\sh_1,\ldots,\sh_k),
\]
then, componentwise,
\[
\bar h_i=\phi(\bar h_{1,i},\ldots,\bar h_{k,i}).
\]
By polynomial growth of $\phi$, there exist constants $\scc,q<\infty$ such that
\[
|\phi(x_1,\ldots,x_k)|
\le
\scc\left(1+\sum_{r=1}^k |x_r|^q\right).
\]
Hence, for every $p\ge 1$,
\[
|\bar h_i|^p
\le
\scc_p\left(1+\sum_{r=1}^k |\bar h_{r,i}|^{pq}\right).
\]
Taking expectations and applying the induction hypothesis with exponent $pq$ gives
\[
\EE[|\bar h_i|^p]\le \scc,
\]
where the constant depends only on the previous moment bounds and on the polynomial-growth
constants of $\phi$, but not on the widths.
\end{proof}

\subsection{Conditional laws}\label{sec:weight-conditional-laws}

The proof of Theorem \ref{thm:main} relies on tracking the laws of the variables introduced by the program,
conditionally on the $\sigma$-algebra generated by the previously defined $\sH$-vars.
The key point is that in both finite-width and infinite-width executions, every newly created $\sG$-var at a
$\matmul$ line is conditionally Gaussian, and its conditional mean/covariance can be computed
recursively from earlier lines.

In this subsection we introduce such recursive formulas in Lemmas \ref{lem:structure-lemma-finite-width} and \ref{lem:structure-lemma-infinite-width}, but first we recall some definitions and results from theory of Gaussian variables.

\subsubsection*{Gaussian variables}
Let $S,T$ be natural numbers.
If $(X,Y)$ is jointly Gaussian with values in $\R^S\times\R^T$,
with mean $\mu_{X,Y} = (\mu_X, \mu_Y)$ and covariance
\[
\Sigma_{(X,Y)}
=
\begin{pmatrix} \Sigma_X & \Sigma_{XY} \\ \Sigma_{YX} & \Sigma_Y \end{pmatrix}
\in \R^{(S+T)\times (S+T)},
\]
then the marginal $Y$ is Gaussian and, conditionally on $Y=y$,
the law of $X$ is Gaussian with
\begin{equation}\label{eq:conditional-gaussian}
  \mu_{X|Y=y} = \mu_X + \Sigma_{XY}\Sigma_Y^{+}(y - \mu_Y),
  \qquad
  \Sigma_{X|Y=y} = \Sigma_{(X,Y)}/\Sigma_Y
               = \Sigma_{X} - \Sigma_{XY}\Sigma_Y^{+}\Sigma_{YX},
\end{equation}
where $\Sigma_Y^+$ denotes the Moore-Penrose pseudoinverse and
$\Sigma_{(X,Y)}/\Sigma_Y$ is the generalized Schur complement
\eqref{eq:schur}.
In particular, $y\mapsto\mu_{X|Y=y}$ is affine and
$y\mapsto\Sigma_{X|Y=y}$ is constant.
These properties also characterize joint Gaussianity \cite{rao1973linear}:

\begin{lemma}\label{lem:conditional-gaussian}
Let $(X, Y)$ be a random variable in $\R^{S}\times \R^{T}$ such that
$Y$ is Gaussian and, for a.e.\ $y\in\R^T$, the conditional law of $X$
given $Y=y$ is Gaussian with $y\mapsto\mu_{X|Y=y}$ affine and
$y\mapsto\Sigma_{X|Y=y}$ constant.
Then $(X,Y)$ is jointly Gaussian.
\end{lemma}

The formulas \eqref{eq:conditional-gaussian} extend verbatim to
matrix-valued Gaussian variables, once covariances are interpreted
as $4$-index tensors.
Concretely, fix $S_1,S_2,T_1,T_2\in\N$ and let $(X,Y)$ be jointly
Gaussian with values in $\R^{S_1\times S_2}\times\R^{T_1\times T_2}$.
Viewing $X$ and $Y$ as vectors of lengths $S_1 S_2$ and $T_1 T_2$
respectively, all standard vector-Gaussian theory applies,
Regarding $X$ and $Y$ as matrices the cross-covariance
\[
  \Sigma_{XY}
  = \E\bigl[(X-\mu_X)\otimes(Y-\mu_Y)\bigr]
  \in \R^{S_1\times S_2\times T_1\times T_2}
\]
is now a $4$-index tensor, where $\otimes$ denotes the Kronecker product, that is, the outer product of
the two matrices regarded as vectors (i.e.\
$(\Sigma_{XY})_{i_1 i_2 j_1 j_2}
= \E[(X-\mu_X)_{i_1 i_2}(Y-\mu_Y)_{j_1 j_2}]$).
The products appearing in \eqref{eq:conditional-gaussian} are then
tensor contractions over the $Y$-indices: given additionally
$Z$ with values in $\R^{U_1\times U_2}$, the product
$\Sigma_{XY}\Sigma_{YZ}$
contracts the $(T_1,T_2)$ indices shared by
$\Sigma_{XY}\in\R^{S_1\times S_2\times T_1\times T_2}$ and
$\Sigma_{YZ}\in\R^{T_1\times T_2\times U_1\times U_2}$,
yielding an element of $\R^{S_1\times S_2\times U_1\times U_2}$.
Equivalently, identifying each matrix space with the corresponding
vector space via lexicographic ordering of indices, the tensor
contraction coincides with the usual matrix product on the
vectorized representatives, so \eqref{eq:conditional-gaussian}
holds.

In particular, given a Gaussian random matrix $W\in\R^{n\times m}$
and a deterministic vector $h\in\R^m$, the product $Wh$ is
Gaussian with
\[
  \mu_{Wh} = \mu_W h,
  \qquad
  \Sigma_{Wh} = \Sigma_W(hh^\top) \in \R^{n\times n},
\]
where the right-hand side uses the tensor contraction just described
(contracting the $m$-dimensional index of $\Sigma_W\in\R^{n\times m\times n\times m}$
against $hh^\top\in\R^{m\times m}$, yielding $\R^{n\times n}$).
The joint variable $(W, Wh) \in \R^{(n+1)\times m}$ is Gaussian, and
\eqref{eq:conditional-gaussian} gives
\[
  \mu_{W|Wh=g}
  = \mu_W + \Sigma_{W,Wh}\Sigma_{Wh}^{+}(g - \mu_W h),
  \qquad
  \Sigma_{W|Wh=g}
  = \Sigma_W - \Sigma_{W,Wh}\Sigma_{Wh}^{+}\Sigma_{Wh,W}.
\]
In the separable covariance case $\Sigma_W = L\otimes R$
(where $L\in\R^{n\times n}$, $R\in\R^{m\times m}$),
the constituent tensors factorize explicitly:
\[
  \Sigma_{W,Wh} = L\otimes(Rh),
  \qquad
  \Sigma_{Wh}   = L\,(h^\top R h),
  \qquad
  \Sigma_{Wh,W} = L\otimes(h^\top R).
\]
Substituting into the conditioning formulas yields
\begin{align}
  \label{eq:update_muW}
  \mu_{W|Wh=g}
  &= \mu_W + (g - \mu_W h)\,(h^\top R h)^{+}(Rh)^\top, \\
  \label{eq:update_RW}
  \Sigma_{W|Wh=g}
  &= L\otimes\bra{R - Rh\,(h^\top Rh)^{+}h^\top R}.
\end{align}
In particular, the separable structure of $\Sigma_W$ is preserved
under conditioning on $Wh=g$.

Recall that since $\Sigma_W$ is a covariance tensor $R$ is positive semidefinite. Let us notice that, since $h^\top Rh \ge 0$ is scalar, the pseudoinverse $(h^\top Rh)^+$ is actually the inverse if $h^\top Rh>0$ and $0$ otherwise. 
This is quite expected since in the latter case $Wh$ is constant and  conditioning on $Wh=g$ leaves the law of $W$ unchanged. 

Let us remark that the formulas above do not require invertibility of $L$. 
Indeed, if $L$ is not invertible, terms of the form $LL^+L$ appear in the Schur complement, but it holds $LL^+L=L$.

\subsubsection*{Finite-width execution}
To describe the conditional law of $\sG$-vars, it is useful to describe first the conditional law of $\sA$-vars. 
Indeed, each $\sA$-var $\sW: \sA(n,m)$ is initialized as a Gaussian random matrix $W\in\R^{n\times m}$ with separable covariance $\Sigma_W = \frac 1 m\Id_n \otimes \Id_m$. 
Given a $\netsor$ program, denote by $\cH_\ell$ the set of $\sH$-vars defined up to line $\ell$.
For the $\sigma$-algebra
$\cF_\ell:=\sigma(\cH_\ell)$,
 conditioning amounts to linear constraints of the form $g=Wh$, where
$g,h\in\cH_\ell$. 
Hence, the conditional law of $W$ remains Gaussian, with updates given by applications of \eqref{eq:update_muW} and \eqref{eq:update_RW} at each $\matmul$ operation. 
In particular, since $\sW:\sA(n,m)$ has separable initial covariance, then we also have
\[
\Sigma_{W\mid \cF_\ell}=\cL_{W\mid\cF_\ell}\otimes \cR_{W\mid\cF_\ell},
\qquad \ell\ge 0,
\]
with suitable $\cL_{W\mid\cF_\ell}$ and $\cR_{W\mid\cF_\ell}$, that we describe below in \eqref{eq:variance-conditional}.

We start with explicitly writing the update rules for $\matmul$ lines.
If line $\ell$ is a $\matmul$ operation $\sg:=\sW\sh$ with finite-width execution $g = Wh$, then $h$ is
$\cF_{\ell-1}$-measurable $\mu_{W\mid \cF_{\ell}}$, $\cL_{W\mid \cF_{\ell}}$ and $\cR_{W\mid \cF_{\ell}}$ are special cases of \eqref{eq:update_muW} and \eqref{eq:update_RW}. 
These read as follows
\begin{equation}\label{eq:mean-conditional}
 \mu_{\sW\mid \cF_\ell}
=
 \mu_{\sW\mid \cF_{\ell-1}}+
\bra{g- \mu_{\sW\mid \cF_{\ell-1}}\,  h}\;
\bra{h^\top \cR_{\sW\mid \cF_{\ell-1}}\, h}^{+}\;
\bra{h^\top  \cR_{\sW\mid \cF_{\ell-1}}}.
\end{equation}
The left part of the covariance tensor remains unchanged:
\begin{equation}
\cL_{\sW\mid\cF_{\ell}} =\cL_{\sW\mid\cF_{\ell-1}} = \frac 1 m \Id_n,
\end{equation}
while the right part satisfies:
\begin{equation}\label{eq:variance-conditional}
 \cR_{\sW\mid \cF_\ell}
=
 \cR_{\sW\mid \cF_{\ell-1}}-
\bra{\cR_{\sW\mid \cF_{\ell-1}}\,  h}\;
\bra{ h^\top  \cR_{\sW\mid \cF_{\ell-1}}\,  h}^{+}\;
\bra{ h^\top  \cR_{\sW\mid \cF_{\ell-1}}}.
\end{equation}
Moreover different $\sA$-vars remain independent.

We obtain as a consequence the conditional law of $\sG$-vars. 
Indeed, these are defined only via $\matmul$ operations: therefore, if line $\ell$ is as above a $\matmul$ operation $\sg:=\sW\sh$ with finite-width execution $g = Wh$, then $h$ is
$\cF_{\ell-1}$-measurable and, conditionally on $\cF_{\ell-1}$,  $g$ is Gaussian with
\begin{equation}\label{eq:mu_g_conditional}
\mu_{g\mid \cF_{\ell-1}}=\mu_{W\mid \cF_{\ell-1}}\,h,
\end{equation}
and
\begin{equation}\label{eq:sigma_g_conditional}
\Sigma_{g\mid \cF_{\ell-1}}
=\cL_{W\mid \cF_{\ell-1}}\;\bra{h^\top \cR_{W\mid \cF_{\ell-1}}\,h}.
\end{equation}

We notice that $\nonlin$ lines do not enlarge the $\sigma$-algebras, hence $\cF_\ell=\cF_{\ell-1}$ for such lines, hence conditional laws remain unchanged in such case. 
Thus, conditional laws are algebraically characterized by following the update rules \eqref{eq:mean-conditional} and \eqref{eq:variance-conditional} along the program lines. 
However, some structural properties are more evident by providing also a geometric description. To this aim, for any program line $\ell$, we introduce a Gram matrix associated to each $\sW: \sA(n,m)$, defined in the following way. 
We collect all the $\matmul$ lines before (possibly including) line $\ell$, in which $\sW$ appears:
\[
\sg^{(1)}=\sW\sh^{(1)},\ \dots,\ \sg^{(k)}=\sW\sh^{(k)},
\]
and we collect such $\sG$-vars and $\sH$-vars in the (random) matrices, respectively in $\R^{n\times k}$ and $\R^{m\times k}$,
\begin{equation}
G_{\sW}^{(\ell)} :=
\bra{ g^{(i)} }_{1\le i\le k}, \quad
H_{\sW}^{(\ell)} :=
\bra{ h^{(i)} }_{1\le i\le k},
\end{equation}
so that the  Gram matrix (associated to the $\sh$-vars for $\sW$ up to line $\ell$) in $\R^{k \times k}$, is given by
\begin{equation}
    \label{eq:gram-finite-width}
H_{\sW}^{(\ell), \top }  H_{\sW}^{(\ell)}  =\bra{h^{(i), \top } h^{(j)}}_{1\le i,j\le k}, 
\end{equation}
and the orthogonal projection on the (random) subspace in $\R^m$, spanned by $\cur{ h^{(i)} }_{1\le i\le k}$, is given by
\begin{equation}
    \label{eq:projection-finite-width}
 H_{\sW}^{(\ell)}   \bra{ H_{\sW}^{(\ell), \top }  H_{\sW}^{(\ell)} }^{+} H_{\sW}^{(\ell) ,\top} =  H_{\sW}^{(\ell)}{H_{\sW}^{(\ell)}}^+ \in \R^{m \times m}.
\end{equation}
With this notation, we will prove the structure lemma \ref{lem:structure-lemma-finite-width} for conditional laws.

\begin{example}
  
Let \(\sW:A(2,3)\), and let its finite-width execution be a Gaussian matrix
\(W\in\mathbb R^{2\times 3}\) with independent entries of variance \(1/3\).
Suppose that, before line \(\ell\), the same A-var has already appeared in two
MatMul lines
\[
 \sg_1=\sW\sh_1,\qquad \sg_2=\sW\sh_2,\qquad \sh_1,\sh_2\in\mathbb R^3,
\]
and assume, for simplicity, that the finite executions \(h_1,h_2\) are linearly independent. After
conditioning on \(\cF_{\ell-1}\), the vectors \(h_1,h_2,g_1,g_2\) are fixed.
Using the matrices
\[
H_W^{(\ell-1)}=(h_1,h_2),\qquad
G_W^{(\ell-1)}=(g_1,g_2),\qquad
{H_W^{(\ell-1)}}^+
=\big((H_W^{(\ell-1)})^\top H_W^{(\ell-1)}\big)^{-1}(H_W^{(\ell-1)})^\top .
\]
we can decompose the finite execution of a new $\matmul$ line \(\sg=\sW\sh\), into the \(\mathcal F_{\ell-1}\)-measurable part
(the one already ``tested'' by \(W\)) and its orthogonal residual:
\begin{align*}
h&= H_W^{(\ell-1)}{H_W^{(\ell-1)}}^+ h + \bra{\Id_3 - H_W^{(\ell-1)}{H_W^{(\ell-1)}}^+}h\\
&= H_W^{(\ell-1)}\alpha+r\\
&=\alpha_1h_1+\alpha_2h_2+r,\qquad \text{where}\\
\alpha=\big(H_W^{(\ell-1)}\big)^+h=(\alpha_1,\alpha_2)^\top,
\qquad&
R_{W\mid\mathcal F_{\ell-1}}
=\Id_3-H_W^{(\ell-1)}\big(H_W^{(\ell-1)}\big)^+,\qquad
r=R_{W\mid\mathcal F_{\ell-1}}h .
\end{align*}
Note that here $R_{W\mid\mathcal F_{\ell-1}}$ is the projection on the orthogonal complement of $\text{span}\{h_1,h_2\}$. Applying \(W\) to the above formula allows a similar decomposition of $g$
\[
g=Wh=\alpha_1Wh_1+\alpha_2Wh_2+Wr
=\alpha_1g_1+\alpha_2g_2+Wr .
\]
where the first two terms are \(\mathcal F_{\ell-1}\)-measurable (i.e., given under the conditioning), while \(Wr\) is fresh Gaussian noise. Based on this decomposition, the conditional distribution of $g$ is given by
\begin{equation}\label{eq:ex_condg}
 g\mid\mathcal F_{\ell-1}
 \sim
 \mathcal N\!\left(
 G_W^{(\ell-1)}\big(H_W^{(\ell-1)}\big)^+h,
 \frac{I_2}{3}\,h^\top R_{W\mid\mathcal F_{\ell-1}}h
 \right).
\end{equation}
This computation foreshadows the content of the structure lemmas (Lemma~\ref{lem:structure-lemma-finite-width} and Lemma~\ref{lem:structure-lemma-infinite-width}) below, characterizing in full generality the information that previous uses of the same A-var provide on how
\(W\) acts on \(\operatorname{span}\{h_1,h_2\}\). As outlined above, these lemmas show that a new product \(Wh\) is predictable
on the projection of \(h\) onto this span, but remains Gaussian through the residual
component orthogonal to it, resulting in the formula \eqref{eq:ex_condg}. If the residual vanishes, then \(h\in\operatorname{span}\{h_1,h_2\}\), so
\(Wh=\alpha_1g_1+\alpha_2g_2\). In this case, the execution of the program can be equivalently rewritten in nondegenerate form by replacing the $\matmul$ line with a $\nonlin$ line, as formalized in Lemma~\ref{lemma:nondegenerate-representative} below.

\end{example}

\begin{lemma}[Structure lemma for finite-width execution]
    \label{lem:structure-lemma-finite-width}
    Consider an $\sA$-var $\sW : \sA(n,m)$ in a $\netsor$ program and let $\cF_\ell = \sigma(\cH_\ell)$ be the $\sigma$-algebra generated by the $\sH$-vars defined up to line $\ell$ by the finite-width execution of the program.
With the notation introduced above, $\cR_{W\mid \cF_{\ell}}$ is the projector on the orthogonal subspace to the span of $\cur{ h^{(i)} }_{1\le i\le k}$,
\[ \cR_{W\mid \cF_{\ell}} = \Id_{m} - H_{\sW}^{(\ell)}  {H_{\sW}^{(\ell)}}^+, \]
and
\[ \mu_{W\mid \cF_{\ell}} = G_{\sW}^{(\ell)} {H_{\sW}^{(\ell)}}^+.\]
If moreover line $\ell$ is $\matmul$ $\sg = \sW \sh$, then, conditionally upon $\cF_{\ell-1}$, $g$ is Gaussian with mean
\[ \mu_{g\mid \cF_{\ell-1}} = G_{\sW}^{(\ell-1)} {H_{\sW}^{(\ell-1)}}^+ h,\]
and covariance
\begin{equation}\label{eq:sigmag}  \Sigma_{g\mid \cF_{\ell-1}} = \frac {\Id_n} m \bra{ h^\top h - h^\top H_{\sW}^{(\ell-1)}   {H_{\sW}^{(\ell-1)}}^+ h },\end{equation}
where $h \in \cF_{\ell-1}$ is the finite-width execution of $\sh$.
\end{lemma}

    In particular, Lemma \ref{lem:structure-lemma-finite-width} implies that $\|\cR_{W\vert \cF_\ell}\|_{\op}\le 1$ whenever the $\ell$-th line is a $\matmul$ instruction.

\begin{proof}

    We reason by induction on $\ell$.
    Let $k$ be the number of $\matmul$ lines involving $\sW$ up to line $\ell$.
    If the $\ell$-th line is $\nonlin$, then there is nothing to prove.
    Hence, assume that the $\ell$-th line is $\sg^{(k+1)} = \sW \sh^{(k+1)}$.
    Write $h = h^{(k+1)}$ and $g = g^{(k+1)}$ for simplicity,
    and $\mu^{(\ell)} = \mu_{\sW\mid \cF_\ell}$, 
    $\cR^{(\ell)} = \cR_{\sW\mid \cF_\ell}$ for each $\ell \le L$.
    For the sake of clarity and since $\sW$ is fixed, we will also omit the dependence of $\sW$ in the block matrices $H^{(\ell)}_\sW$ and $G^{(\ell)}_\sW$.
    We proceed to prove the claim for $\mu^{(\ell)}$ and $\cR^{(\ell)}$ separately.

    \vspace{0.5cm}

    \emph{Proof for $\cR^{(\ell)}$:}

    We begin by studying the case $h^\top \cR^{(\ell-1)} h = 0$.
 This implies that $h$ is in the span of $\{h^{(1)},\dots,h^{(k)}\}$ and $\bra{h^\top \cR^{(\ell-1)} h}^+ = 0$, yielding
 \[\cR^{(\ell)} = \cR^{(\ell-1)} = \Id_{m} - H^{(\ell-1)}  {H^{(\ell-1)}}^+.\]
 Note that since the column spaces of $\{h^{(1)},\dots,h^{(k)}\}$ and $\{h^{(1)},\dots,h^{(k)},h\}$ coincide, 
 we have $H^{(\ell)}{H^{(\ell)}}^+ = H^{(\ell-1)}{H^{(\ell-1)}}^+$, so the claim holds in the degenerate case.
We now consider the case $h^\top \cR^{(\ell-1)} h > 0$.
Denote by $P$ the orthogonal complement to the span of $\{h^{(1)},\dots,h^{(k)}\}$ in $\R^m$.
Recalling the update \eqref{eq:variance-conditional},
\[
\cR^{(\ell)}
= \cR^{(\ell-1)}
\left(
\Id_m - h h^\top \cR^{(\ell-1)}(h^\top \cR^{(\ell-1)} h)^{-1}
\right),
\]
is immediate under the induction hypothesis that $\cR^{(\ell)}$ is symmetric and idempotent.
Interpreting $\cR^{(\ell)}$ as a linear operator on $\R^m$, we have
$\im \cR^{(\ell)} \subseteq \im \cR^{(\ell-1)} = P$.
Moreover, for every $x \in \R^m$,
\[
h^\top \cR^{(\ell)} x
= h^\top \cR^{(\ell-1)} x
- h^\top \cR^{(\ell-1)} h (h^\top \cR^{(\ell-1)} h)^{-1} h^\top \cR^{(\ell-1)} x 
= 0,
\]
so that
$\im \cR^{(\ell)} \subseteq P$.
Conversely, if $y \in P$, then
$\cR^{(\ell-1)} y = y$ and $h^\top y = 0$, hence
\[
\cR^{(\ell)} y = \cR^{(\ell-1)} y -
\cR^{(\ell-1)} h  (h^\top \cR^{(\ell-1)} h)^{-1} h^\top \cR^{(\ell-1)} y
= y .
\]
Thus $P \subseteq \im \cR^{(\ell)}$. 
The uniqueness of the orthogonal projection implies the first claim.

\vspace{0.5cm}

    \emph{Proof for $\mu^{(\ell)}$}:

    Again, we study first the case $h^\top \cR^{(\ell-1)} h = 0$.
    In this case, the update \eqref{eq:mean-conditional} reduces to $\mu^{(\ell)} = \mu^{(\ell-1)}$.
    In effect, note that $G^{(i)} = WH^{(i)}$ for each $i \le \ell$. 
    Therefore, by reasoning as in the previous case:
    \[G^{(\ell)} {H^{(\ell)}}^+ = WH^{(\ell)} {H^{(\ell)}}^+ =WH^{(\ell-1)} {H^{(\ell-1)}}^+ = G^{(\ell-1)} {H^{(\ell-1)}}^+.\]
    This proves the claim in the degenerate case.
    Let's now assume that $h^\top \cR^{(\ell-1)} h > 0$.
The recursion \eqref{eq:mean-conditional} can be written as a rank-one correction
\[
\mu^{(\ell)}
=
\mu^{(\ell-1)}
+
\bigl(g-\mu^{(\ell-1)}h\bigr)q_\ell^\top,
\]
where
\[
q_\ell^\top
:=
\bigl(h^\top R^{(\ell-1)}h\bigr)^{-1}
h^\top R^{(\ell-1)} .
\]

By induction hypothesis,
\[\mu^{(\ell)} = G^{(\ell-1)}{H^{(\ell-1)}}^+ + (g-G^{(\ell-1)}{H^{(\ell-1)}}^+h)q_\ell.\]
On the other hand, since $G^{(\ell)} = \bra{G^{(\ell-1)}, g}$ and $H^{(\ell)} = \bra{H^{(\ell-1)}, h}$, 
and since $h$ is not in the span of $\cur{h^{(i)}}_{1\le i\le k}$ by assumption, 
we have $(\Id_m-H^{(\ell-1)}{H^{(\ell-1)}}^+)h \ne 0$, 
so the first branch of Greville's formula \eqref{eq:greville} can be applied with $A = H^{(\ell-1)}$ and $a=h$.
The corresponding row vector $b \in \R^{1\times m}$ is:
\[
b = (h-H^{(\ell-1)}{H^{(\ell-1)}}^+h)^+ = (\cR^{(\ell-1)}h)^+.
\]
Since $\cR^{(\ell-1)}$ is an orthogonal projector and $\cR^{(\ell-1)}h\ne 0$, the pseudoinverse is given by:
\[
b = \frac{(\cR^{(\ell-1)}h)^\top}{h^\top\cR^{(\ell-1)}h}=q_\ell^\top.
\]
Then, Greville's formula \eqref{eq:greville} yields
\[{H^{(\ell)}}^+ = \begin{pmatrix}
    {H^{(\ell-1)}}^+(\Id_m - h q_\ell) \\ q_\ell^\top
\end{pmatrix},\]
Therefore,
\begin{align}
G^{(\ell)}{H^{(\ell)}}^+ &= G^{(\ell-1)}{H^{(\ell-1)}}^+(\Id_m - h q_\ell^\top) + gq_\ell^\top \\
&= G^{(\ell-1)}{H^{(\ell-1)}}^+ + (g-G^{(\ell-1)}{H^{(\ell-1)}}^+h)q_\ell^\top\\
&= \mu^{(\ell)}.
\end{align}
This concludes the proof of the second claim.

Lastly, the formulas for the conditional mean and covariance of $g$ follow directly from the first two claims and the general update formulas \eqref{eq:mu_g_conditional} and \eqref{eq:sigma_g_conditional},
substituting $\cL_{W\mid \cF_{\ell-1}} = \frac 1 m \Id_n$ in the latter.
\end{proof}

\subsubsection*{Infinite-width execution}

In the infinite-width execution, $\sA$-vars are not realized as actual random matrices. Nevertheless,
it is convenient to associate to each $\sA$-var $\sW: \sA(n,m)$ a pair of \emph{barred} parameters
$\bmu_{\sW\mid \bF_\ell}$ and $\bSigma_{\sW\mid \bF_\ell}$ playing the role of conditional mean and covariance. 
As in the finite-width execution, the covariance update preserves separability; we write
\[
\bSigma_{\sW\mid \bF_\ell}=\bL_{\sW\mid \bF_\ell}\otimes \bR_{\sW\mid \bF_\ell}.
\] 
These quantities evolve by the same regression update maps as in the finite-width
execution, but with all contractions computed using the infinite-width pairing $\ldot$ defined in \eqref{eq:infinite-width-product}.  Concretely, when line $\ell$ is $\matmul$ $\sg=\sW\sh$ with executions
$\bar g$, $\bar h$ and $\bF_\ell:=\sigma(\bar{\cH}_\ell)$, we set
\begin{equation}\label{eq:mean-bar-conditional}
\bmu_{\sW\mid \bF_\ell}
=
\bmu_{\sW\mid \bF_{\ell-1}}
+
\bra{\bar g-\bmu_{\sW\mid \bF_{\ell-1}}\, \bar h}\;
\bra{\bar h\ldot \bR_{\sW\mid \bF_{\ell-1}}\,\bar h}^{+}\;
\bra{\bar h\ldot \bR_{\sW\mid \bF_{\ell-1}}},
\end{equation}
and 
where $\cL_{\sW\mid \bF_\ell} = \cL_{\sW\mid \bF_{\ell-1}} = \Id_n$, and  
\begin{equation}\label{eq:variance-bar-conditional}
\bR_{\sW\mid \bF_\ell}
=
\bR_{\sW\mid \bF_{\ell-1}}
-
\bra{\bR_{\sW\mid \bF_{\ell-1}}\, \bar h}\;
\bra{\bar h\ldot \bR_{\sW\mid \bF_{\ell-1}}\, \bar h}^{+}\;
\bra{\bar h\ldot \bR_{\sW\mid \bF_{\ell-1}}}.
\end{equation}
Such barred parameters can be thought as mere notational convenience, in order to describe the actual conditional law of $\sG$-vars: indeed,
conditionally on $\bF_{\ell-1}$, the new $\bar g\in\R^{n}$ is Gaussian with mean
\begin{equation}
    \label{eq:mu_g_bar_conditional}
    \mu_{\bar g \mid\bF_{\ell-1} } = \bmu_{\sW\mid\bF_{\ell-1}}\bar h,
\end{equation}
and covariance
\begin{equation}
    \label{eq:sigma_g_bar_conditional}
\Sigma_{\bar g \mid\bF_{\ell-1}} = \bL_{\sW\mid\bF_{\ell-1}}\bra{\bar h\ldot \bR_{\sW\mid\bF_{\ell-1}}\bar h}.
\end{equation} 
We notice  that also in the infinite-width execution, \textsc{NonLin} lines do not enlarge the $\sigma$-algebras,
 hence $\bF_\ell=\bF_{\ell-1}$ for such lines, and conditional laws remain unchanged in such case.

To provide a geometric description for the infinite-width conditional laws, we proceed analogously to the finite-width execution. We collect all the $\matmul$ lines before (possibly including) line $\ell$, in which $\sW$ appears:
\[
\sg^{(1)}=\sW\sh^{(1)},\ \dots,\ \sg^{(k)}=\sW\sh^{(k)}.
\]
We define the collections of their infinite-width executions $\bar g^{(i)} \in \R^n$ and $\bar h^{(i)} \in L^2(\R^m)$ via the formal operators:
\[
\bG_{\sW}^{(\ell)} := \bra{\bar g^{(i)}}_{1\le i\le k}, \quad \bH_{\sW}^{(\ell)} := \bra{\bar h^{(i)}}_{1\le i\le k}.
\]
Analogously to \eqref{eq:gram-finite-width}, we define the Gram matrix in the infinite-width execution using the infinite-width pairing $\ldot$ defined in Equation \eqref{eq:infinite-width-product}:
\[
\bH^{(\ell)}_{\sW}\ldot\bH^{(\ell)}_{\sW} := \bra{\bar h^{(i)} \ldot \bar h^{(j)}}_{1\le i,j\le k}.
\]
In view of \eqref{eq:sigma-gg-mf}, the covariance matrix of the stacked vector $(\bar g^{(i)})_{i=1, \ldots, k}$ is given by $(\bH^{(\ell)}_{\sW}\ldot\bH^{(\ell)}_{\sW}) \otimes \Id_n$.
The orthogonal projection on the span of $\{\bar h^{(i)}\}_{1\le i \le k}$ in $L^2(\R^m)$ is written compactly as:
\[
  \bH_{\sW}^{(\ell)} \bra{ \bH_{\sW}^{(\ell)} \ldot \bH_{\sW}^{(\ell)} }^{+} \bH_{\sW}^{(\ell)} \ldot = \bH_{\sW}^{(\ell)}{\bH_{\sW}^{(\ell)}}^+.
\]

A proof analogous to Lemma \ref{lem:structure-lemma-finite-width} yields the same structure for the infinite-width execution, 
with the only difference that all products are computed using the infinite-width pairing $\ldot$ instead of the Euclidean inner product.
In particular, Greville's formula \eqref{eq:greville} still holds in this Hilbert space setting by choosing an orthonormal basis $\cur{u_i}_{1\le i\le k}$ of the span of $\cur{\bar h^{(i)}}_{1\le i\le k}$ in $L^2(\R^m)$,
as detailed at the beginning of the present section.

\begin{lemma}[Structure lemma for infinite-width execution]
    \label{lem:structure-lemma-infinite-width}
    Consider an $\sA$-var $\sW : \sA(n,m)$ in a $\netsor$ program and let $\bF_\ell$ be the $\sigma$-algebra generated by the $\sH$-vars defined up to line $\ell$ by the infinite-width execution of the program.
With the notation introduced above, $\bR_{W\mid \bF_{\ell}}$ is the projector on the orthogonal subspace to the span of $\cur{ \bar h^{(i)} }_{1\le i\le k}$,
\[ \bR_{W\mid \bF_{\ell}} = \Id_{m} - \bH_{\sW}^{(\ell)}  \bH_{\sW}^{(\ell),+} . \]
and
\[ \bmu_{W\mid \bF_{\ell}} = \bG_{\sW}^{(\ell)} \bH_{\sW}^{(\ell),+}.\]
If moreover line $\ell$ is $\matmul$ $\sg = \sW \sh$, then, conditionally upon $\bF_{\ell-1}$, $\bar g$ is Gaussian with mean
\[ \mu_{\bar g\mid \bF_{\ell-1}} =\bG_{\sW}^{(\ell-1)} \bH_{\sW}^{(\ell-1),+} \bar h\]
and covariance
\begin{equation}\label{eq:sigmag-infty}  
    \Sigma_{\bar g\mid \bF_{\ell-1}} = 
    \Id_n \bra{\bar h \ldot \bar h - \bar h \ldot \bH_{\sW}^{(\ell-1)} \bH_{\sW}^{(\ell-1),+} \bar h},
\end{equation}
where $\bar h \in \bF_{\ell-1}$ is the infinite-width execution of $\sh$.
\end{lemma}

\color{black}

\begin{remark}[Input parents in the basic $\netsor$ language]
In a basic $\netsor$ program, if a $\matmul$ line \(\sg=\sW\sh\) has parent dimension
\(m\in\nin\), then the finite execution \(h\) of $\sh$ is obtained from deterministic input variables by ordinary
$\nonlin$ operations. Hence its finite-width and infinite-width executions coincide.
In particular, the corresponding Gram quantities are structural. This statement is used
only in the proof for basic $\netsor$ programs. In the extended $\netsork$ language, input-dimensional
parents may also depend on scalar random variables, and the resulting finite-width fluctuations
are controlled separately in the $\nonlinpl / \scal$ cases of the induction step.
\end{remark}

\subsection{Program equivalence and reduction to non-degeneracy}
\label{subsec:equivalence-reduction}

In this subsection we introduce a notion of equivalence between \netsor\ programs, and we use it to show that every program is equivalent to one in which all $\matmul$ lines have nondegenerate conditional Gaussian outputs, in the infinite-width execution. 
This reduction is useful in the proof of the main theorem, since it allows us to work with densities and avoid technicalities related to degenerate Gaussians and non-invertible covariance matrices.

\begin{definition}[Program equivalence]
Let $T$ and $T'$ be two \netsor\ programs of length $M$, and denote by
\[
\sh^{(1)},\ldots,\sh^{(M)}
\qquad\text{and}\qquad
\sh'^{(1)},\ldots,\sh'^{(M)}
\]
their $\sH$-vars. We say that $T$ and
$T'$ are \emph{infinite-width-equivalent} if there exists a bijection
\[
\pi:\{\sh^{(1)},\ldots,\sh^{(M)}\}\longrightarrow
\{\sh'^{(1)},\ldots,\sh'^{(M)}\}
\]
such that, writing $\pi(\sh^{(i)})=\sh'^{(\pi(i))}$, the corresponding infinite-width
executions satisfy
\[
\bigl(\bar h^{(1)},\ldots,\bar h^{(M)}\bigr)
\stackrel{\rm law}{=}
\bigl(\bar h'^{(\pi(1))},\ldots,\bar h'^{(\pi(M))}\bigr).
\]
We say that $T$ and $T'$ are \emph{equivalent} if the same bijection also identifies the
finite-width executions, namely
\[
\bigl(h^{(1)},\ldots,h^{(M)}\bigr)
\stackrel{\rm law}{=}
\bigl(h'^{(\pi(1))},\ldots,h'^{(\pi(M))}\bigr).
\]
\end{definition}

We stress that in the definition above the bijection is not required to preserve the syntactic subtype: a $\sG$-var in one program may correspond to an $\sH$-var which is not a $\sG$-var in the other.

\begin{lemma}[Non-degenerate representative]
\label{lemma:nondegenerate-representative}
Let $T$ be a \netsor\ program whose $\nonlin$ functions are Lipschitz continuous.
Then $T$ is equivalent to a \netsor\ program $T'$ such that, in the infinite-width
execution of $T'$, every $\matmul$ line has non-degenerate conditional Gaussian
output: if $\sg$ is generated at line $\ell$ of $T'$, then
\[
\bar g \mid \bF_{\ell-1}
\]
is a non-degenerate Gaussian random vector.
\end{lemma}

\begin{remark}[Gram formulation of non-degeneracy]
\label{rem:gram-nondegeneracy}
By the structure lemma for the infinite-width execution, the conclusion of
Lemma~\ref{lemma:nondegenerate-representative} is equivalent to saying that no
$\matmul$ line contracts a parent which is already in the infinite-width span of the
previous parents associated with the same $\sA$-var. More explicitly, if the line is
\[
\sg=\sW\sh
\]
and the previous uses of the same $\sA$-var $\sW$ have parents
$\sh^{(1)},\ldots,\sh^{(k)}$, then the conditional covariance of $\bar g$ is
non-degenerate precisely when
\[
\bar h \ldot \bR_{\sW|\bF_{\ell-1}}\bar h >0,
\]
or equivalently when the enlarged Gram matrix
\[
\overline H^{(\ell)}_{\sW}\ldot \overline H^{(\ell)}_{\sW}
\]
has rank one more than
\[
\overline H^{(\ell-1)}_{\sW}\ldot \overline H^{(\ell-1)}_{\sW}.
\]
In particular, after applying Lemma \ref{lemma:nondegenerate-representative}, the Gram matrices for each $\sA$-var are
invertible along the program.
\end{remark}

Before proving Lemma \ref{lemma:nondegenerate-representative} we present an auxiliary result on the geometrical properties of the covariance tensor update \eqref{eq:variance-conditional}.

\begin{proposition}
\label{prop:matmul-reduction}
Given a program with line $\ell$ of $\matmul$ type $\sg=\sW\sh$, the following are equivalent:
\begin{enumerate}
\item[\emph{(i)}] $\bar g$ is Gaussian degenerate, conditionally upon $\bF_{\ell-1}$;
\item[\emph{(ii)}] $\bar h\ldot \bR_{\sW\mid \bF_{\ell-1}}\,\bar h = 0$;
\item[\emph{(iii)}] there exist (deterministic) $(\alpha_i)_{i=1, \ldots, k}$ and previously defined
$\sG$-vars $(\sg^{(i)})_{i=1, \ldots, k} \in \cH_{\ell-1}$ using the same $\sA$-var $\sW$ such that, replacing the $\matmul$ line $\sg=\sW\sh$ by a $\nonlin$ instruction:
\[
\sg' := \sum_{i=1}^k \alpha_i \sg^{(i)}
\]
yields an infinite-width-equivalent program under the correspondence $\sg  \mapsto \sg'$ and keeps the other variables fixed.
\end{enumerate}
Moreover, if all the $\sG$-vars in $\cH_{\ell-1}$ are non-degenerate in the infinite-width execution and non-linearities are Lipschitz continuous, then (iii) yields an equivalent program under the same correspondence. 
\end{proposition}

\begin{remark}
    The equivalence (iii) changes the filtration of the $\netsor$ program by substituting the $\sG$-var $\sg$ with the $\sH$-var $\sg'$. 
    Informally, this provides a way to eliminate degenerate conditional infinite-width executions of $\sG$-vars by skipping ``redundant'' levels of the filtration and replacing the $\sG$-var with a linear combination of previous $\sG$-vars, which can be implemented as a $\nonlin$ line with a linear map.
\end{remark}

\begin{proof}[Proof of Lemma \ref{lemma:nondegenerate-representative}]
We construct a finite sequence of programs $T_0, T_1, \dots, T_N$, starting with $T_0 = T$, where each program is equivalent to the previous one and has strictly fewer degenerate $\matmul$ lines.
Iterate through the lines of the current program $T$ sequentially, from $\ell = 1$ to $M$. 
Suppose line $\ell$ is the first $\matmul$ instruction $\sg^{(\ell)} = \sW \sh$ that is degenerate in the sense of Proposition~\ref{prop:matmul-reduction},
i.e., $\bar g^{(\ell)}  \mid \bF_{\ell-1}$ is a degenerate random variable. 
Because we proceed sequentially, all $\sG$-vars generated by $\matmul$ lines prior to $\ell$ are already non-degenerate in the infinite-width execution conditionally upon $\bF$. 
We apply Proposition~\ref{prop:matmul-reduction} to replace the instruction at line $\ell$ with the $\nonlin$ instruction $\sg' := \sum_{i=1}^k \alpha_i \sg^{(i)}$, yielding a new program $T_{new}$. 

By Proposition~\ref{prop:matmul-reduction}.(iii), $T_{new}$ is infinite-width-equivalent to the previous program. 
Since the nonlinearities are assumed to be continuous, the final step of Proposition~\ref{prop:matmul-reduction} applies and the replacement yields an equivalent
program.

We update our current program to $T_{new}$ and continue the iteration. 
Since the original program $T$ has a finite number of lines $M$, and each substitution strictly removes a single degenerate $\matmul$ line without introducing new ones, 
this procedure terminates. The final program $T_N := T'$ is equivalent to $T$ by transitivity, and every remaining $\matmul$ line has strictly positive conditional variance.
\end{proof}

\begin{proof}[Proof of Proposition \ref{prop:matmul-reduction}]
\emph{(i)$\Leftrightarrow$(ii).}
By 
Equation \eqref{eq:sigma_g_bar_conditional}
\[
\Sigma_{\bar g\mid \bF_{\ell-1}}
=\bL_{\sW\mid \bF_{\ell-1}}\big(\bar h\ldot \bR_{\sW\mid \bF_{\ell-1}}\bar h\big),
\]
so conditional degeneracy of $\bar g$ given $\bF_{\ell-1}$ is equivalent to vanishing of the scalar factor
$\bar h\ldot \bR_{\sW\mid \bF_{\ell-1}}\bar h$.

\emph{(ii)$\Rightarrow$(iii).}
By 
Lemma \ref{lem:structure-lemma-infinite-width},
the identity $\bar h\ldot \bR_{\sW\mid \bF_{\ell-1}}\bar h=0$ implies that $\bar h$ belongs to
the span of the previous parents $\{\bar h_i:\ \sg^{(i)}=\sW\sh^{(i)} \text{ occurs before line }\ell\}$.
Since only finitely many parents occur before $\ell$, there exist deterministic coefficients
$(\alpha_i)_{i=1}^k$ such that
\begin{equation}\label{eq:linear-reduction}
\bar h=\sum_{i=1}^k \alpha_i \bar h_i\quad \text{almost surely.}
\end{equation}

Using the bilinearity of the covariance, we can expand the covariance matrix of the difference $\bar g - \sum_{i=1}^k \alpha_i \bar g^{(i)}$. 
Applying \eqref{eq:sigma-gg-mf} to each term yields:
\begin{align*}
    \Sigma_{ \bar g -\sum_{i=1}^k \alpha_i \bar g^{(i)} } &= \Sigma_{\bar g} - 2 \sum_{i=1}^k \alpha_i \Sigma_{\bar g, \bar g^{(i)}} + \sum_{i,j=1}^k \alpha_i \alpha_j \Sigma_{\bar g^{(i)}, \bar g^{(j)}} \\
    &= (\bar h \ldot \bar h)\Id_n - 2 \sum_{i=1}^k \alpha_i (\bar h \ldot \bar h_i)\Id_n + \sum_{i,j=1}^k \alpha_i \alpha_j (\bar h_i \ldot \bar h_j)\Id_n.
\end{align*}
By the bilinearity of the infinite-width pairing $\ldot$, this expression factors into:
\[
    \Sigma_{ \bar g -\sum_{i=1}^k \alpha_i \bar g^{(i)} } = \bra{ \bar h - \sum_{i=1}^k \alpha_i \bar h_i }\ldot\bra{ \bar h - \sum_{j=1}^k \alpha_j \bar h_j } \Id_n.
\]
By \eqref{eq:linear-reduction}, the term inside the brackets is zero, meaning this covariance matrix is identically zero. 
Since a centered Gaussian vector with zero variance must be zero almost surely, we obtain $\bar g = \sum_{i=1}^k \alpha_i \bar g^{(i)}$ almost surely,
hence replacing the $\sG$-var $\sg=\sW\sh$ with the $\sH$-var $\sg'=\sum_i\alpha_i \sg^{(i)}$ yields an infinite-width-equivalent program.
To show the last claim, 
we introduce continuous functions $(\phi^{(i)})_{i=1, \ldots, k}$ and $\phi$ such that $\phi^{(i)}( (\bar g^{(j)})_{j=1, \ldots, k'}) = \bar h^{(i)}$ and similarly $\phi$ such that $\phi( (\bar g^{(j)})_{j=1, \ldots, k'}) = \bar h$, where possibly the  set of $\sG$-vars involved $(\sg^{(j)})_{j=1, \ldots, k'}$ is different than $(\bar g^{(i)})_{i=1, \ldots, k}$, but still contained in $\cH_{\ell-1}$, hence they are non-degenerate Gaussians. 
Consider the continuous function $\Phi\colon \R^{k'}\to \R$ given by
\[
\Phi((z_j)_{j=1, \ldots, k'}) := \phi((z_i)_{j=1, \ldots, k'}) - \sum_{i=1}^k \alpha_i \phi^{(i)}((z_j)_{j=1, \ldots, k'}).
\]
Then $\Phi((\bar g^{(j)})_{j=1, \ldots, k'}) = 0$ (componentwise in $\R^n$) Lebesgue-almost surely.
Since the $(\bar g^{(j)})_{j=1, \ldots, k'}$ are non-degenerate, the set $\{z\in\R^{k'}:\ \Phi(z)\ne 0\}$ has null Lebesgue measure, and by continuity of $\Phi$ it is an open set, hence empty. In particular, $\Phi((g^{(j)})_{j=1, \ldots, k'})=0$ almost surely under finite-width execution as well, so the same linear relation \eqref{eq:linear-reduction} holds in finite-width execution. Multiplying by $W$ gives $g'=g$ almost surely.

\emph{(iii)$\Rightarrow$(ii).}
Assume $\bar g=\sum_i\alpha_i \bar g^{(i)}$ where each $\bar g^{(i)}=\sW\bar h^{(i)}$ uses the same $\sW$ and
appears before $\ell$. By Lemma~\ref{lem:structure-lemma-infinite-width}, the conditional covariance of the
stacked block $(\bar g^{(1)},\dots,\bar g^{(k)},\bar g)$ equals a Kronecker product of $I_n$
with the Gram matrix of the parents $(\bar h^{(1)},\dots,\bar h^{(k)},\bar h)$. The linear relation among the
$\bar g$'s forces this Gram matrix to be singular, hence $\bar h$ lies in the span of previous
parents. Equivalently, its orthogonal projection coefficient vanishes:
$\bar h\ldot \bR_{\sW\mid \bF_{\ell-1}}\bar h=0$.
\end{proof}

\section{Proof of main results}
\label{sec:proof-main}

The proof of Theorem~\ref{thm:main} is by induction on the program lines. A key input in
the induction is a quantitative kernel LLN for pairs of $\sH$-vars already constructed
before the current line. To avoid any circularity, we first prove this kernel estimate in
the following conditional form: if the conclusion of Theorem~\ref{thm:main} holds for the
truncated program up to line $r-1$, then the corresponding kernel LLN holds for all
$\sH$-vars generated up to line $r-1$.

This conditional kernel estimate will be used in the induction step from $r-1$ to $r$.
Once the induction proving Theorem~\ref{thm:main} is complete, the same argument gives
Corollary~\ref{coro:kernel-lln} for the full program. We therefore begin with the proof of
the kernel estimate, and then turn to the proof of Theorem~\ref{thm:main}.

For the sake of completeness we recall some elementary inequalities involving the Wasserstein distance that will be used in the proofs of Theorem \ref{thm:main} and Corollary \ref{coro:kernel-lln}.
The reader interested in a general treatment of the Wasserstein distance is referred to \cite{villani_transport}.
Given random variables $(X,X')$, $(Y, Y')$ taking values in $\R^n\times \R^n$, it holds
 \begin{equation}\label{eq:kernel-basteri-tool}
  \W_p\bra{ X^\top X', Y^\top Y' } \le  \bra{ \nor{X'}_{L^{2p}}+ \nor{Y}_{L^{2p}} } \W_{2p}\bra{(X, X'),(Y, Y')},
 \end{equation}
 with $\cost = \cost(p)<\infty$.
Indeed, consider any coupling of $(X,X')$ and $(Y,Y')$. 
By adding and subtracting $Y^\top X'$,
\begin{equation}
  \label{eqaux_kernel-basteri-tool}
  \nor{ X^\top X' - Y^\top Y'}
\le \nor{(X-Y)^\top X'} + \nor{Y^\top (X'-Y')}
\le \|X-Y\|\,\|X'\| + \|Y\|\,\|X'-Y'\|.
\end{equation}
Taking the $L^p$ norm, using the triangle inequality and Cauchy-Schwarz inequality, we obtain \eqref{eq:kernel-basteri-tool}.

Similarly, the following bound holds:
\begin{equation}\label{eq:kernel-basteri-tool2}
  \W_p\bra{ X^\top X', Y^\top Y' } \le  \bra{1+\nor{(Y, Y')}_{L^{2p}} +  \W_{2p}(X',Y')} \W_{2p}\bra{(X, X'),(Y, Y')}.
 \end{equation}
 This is obtained from using the triangle inequality $\|X'\| \le \|X'-Y'\| + \|Y'\|$ in \eqref{eqaux_kernel-basteri-tool} above.

\subsection{Proof of Corollary~\ref{coro:kernel-lln}}
\label{sec:kernel-lln-proof}

Throughout this proof, $\scc<\infty$ denotes a constant that may depend on $p$ and on
the structural constants of the program, but not on the hidden widths. Its value may
change from line to line. Moreover, we can assume $p \ge 2$ without loss of generality, and also that $n \notin \nin$, otherwise the thesis is trivial because both variables are deterministic, noticing that the left hand side is zero in this case.
Let $\sh,\sh':\sH(n)$ be two $\sH$-vars of common length $n$, with finite-width executions
$(h,h')$ and infinite-width executions $(\bar h,\bar h')$. 
Set
\[
K_n := \frac1n h^\top h',
\qquad
\bar K_n := \frac1n \bar h^\top \bar h',
\qquad
K := \E[\bar K_n]=\bar h\ldot \bar h'.
\]
Since $K$ is deterministic,
\[
\|K_n-K\|_{L^p}
=
\W_p(K_n,K).
\]
By the triangle inequality for Wasserstein distances,
\begin{equation}\label{eq:kernel-decomp}
\|K_n-K\|_{L^p}
=
\W_p(K_n,K)
\le
\W_p(K_n,\bar K_n)+\W_p(\bar K_n,K).
\end{equation}
We estimate the two terms separately.

For the first term, apply \eqref{eq:kernel-basteri-tool2} with
\[
X:=\frac{h}{\sqrt n},\quad X':=\frac{h'}{\sqrt n},
\qquad
Y:=\frac{\bar h}{\sqrt n},\quad Y':=\frac{\bar h'}{\sqrt n}.
\]
Then $X^\top X'=K_n$ and $Y^\top Y'=\bar K_n$, and therefore
\[
\W_p(K_n,\bar K_n)
\le
\scc_p
\Bigl(
1+\|(Y,Y')\|_{L^{2p}}+\W_{2p}(X',Y')
\Bigr)
\W_{2p}\bigl((X,X'),(Y,Y')\bigr).
\]
The moment factor $\|(Y,Y')\|_{L^{2p}}$ is bounded uniformly in the widths by the
infinite-width moment bounds. Moreover,
\[
\W_{2p}(X',Y')
\le
\W_{2p}\bigl((X,X'),(Y,Y')\bigr)
\le
\W_{2p}\!\left(
\left(\frac{h^{(\ell)}}{\sqrt{n_\ell}}\right)_{\ell=1,\ldots,M},
\left(\frac{\bar h^{(\ell)}}{\sqrt{n_\ell}}\right)_{\ell=1,\ldots,M}
\right).
\]
Using \eqref{eq:inequality-theorem-main-nngp} with exponent $2p$, we obtain therefore:
\[
\W_{2p}(X',Y')
+
\W_{2p}\bigl((X,X'),(Y,Y')\bigr)
\le
\scc \sum_{m \in \nhidd} \frac1{\sqrt{m}}.
\]
Thus
\begin{equation}\label{eq:kernel-transport-bound}
\W_p(K_n,\bar K_n)
\le
\scc \sum_{m \in \nhidd} \frac1{\sqrt{m}}.
\end{equation}

It remains to estimate the sampling term $\W_p(\bar K_n,K)$. Since $K$ is deterministic,
\[
\W_p(\bar K_n,K)=\|\bar K_n-K\|_{L^p}.
\]
We write
\[
\bar K_n=\frac1n\sum_{i=1}^n U_i,
\qquad
U_i:=\bar h_i\bar h'_i.
\]
By definition of the infinite-width execution, the variables $(U_i)_{i=1}^n$ are i.i.d.,
and $K=\E[U_1]$. Let $\tilde U_i:=U_i-\E[U_i]$ for $i=1,\ldots,n$.

Then Rosenthal's inequality \eqref{eq:rosenthal}
gives, for each $p\ge 2$,
\begin{align*}
\Big\|\frac1n \sum_{i=1}^n \tilde U_i\Big\|_{L^p}&\le \scc(p) \bra{ \bra{\sum_{i=1}^n \frac{\|\tilde U_i\|_{L^2}^2}{n^2}}^{p/2} +\sum_{i=1}^n \frac{\|\tilde U_i\|_{L^p}^p}{n^p}}^{1/p}\\
&\le \scc(p)  \bra{\frac{\|\tilde U_1\|_{L^2}}{\sqrt{n}}+\frac{\|\tilde U_1\|_{L^p}}{n^{1-1/p}}},
\end{align*}
where the constant $\scc(p)$ depends on $p$ only.
Moreover, since $U_1 = \bar h_1\bar h'_1$ has finite moments by Lemma \ref{lem:infinite-width-finite-moments} and $1-1/p\ge 1/2$ for $p\ge 2$,
 we can further simplify:
\[ \Big\|\frac1n \sum_{i=1}^n \tilde U_i\Big\|_{L^p}\le \frac{\scc}{\sqrt n},\]
where $\scc$ depends on the $p$-th moment of the $\tilde U_i$, hence it is a structural constant of the program.

Therefore, \eqref{eq:kernel-decomp} can be estimated with:
\[
\|K_n-K\|_{L^p} \le \scc \sum_{m \in \nhidd} \frac1{\sqrt{m}} + \frac{\scc}{\sqrt n} \le \scc \sum_{m \in \nhidd \cup \nout} \frac1{\sqrt{m}}.
\]

%
%
\subsection{Proof of Theorem~\ref{thm:main}}
\label{subsec:proof-main-theorem}

We prove the theorem by constructing an explicit coupling between the finite-width and
infinite-width executions, line by line. Throughout the proof, constants denoted by $\scc$
may depend on $p$ and on the structural constants of the program, but never on the hidden
widths.

By Lemma~\ref{lemma:nondegenerate-representative}, we may assume without loss of
generality that the program is non-degenerate in the infinite-width execution. In particular,
the Gram matrices appearing in the structure formulas for the conditional laws stay in a
fixed smooth stratum at the infinite-width point, so that the maps involving pseudoinverses
and matrix square roots are differentiable there.

We denote by $L$ the number of lines in the program.
For $r\le L$, define the normalized error vector
\[
\Delta_r
:=
\left(
\frac{h^{(\ell)}-\bar h^{(\ell)}}{\sqrt{n_\ell}}
\right)_{\ell=1,\ldots,r}.
\]
We prove by induction on $r$ that the coupling can be chosen so that
\begin{equation}
\label{eq:main-inductive-bound}
\|\Delta_r\|_{L^p}
\le
\scc \sum_{m \in \nhidd(r)} \frac1{\sqrt{m}}, 
\end{equation}
where $\nhidd(r)$ denotes the hidden widths of the program restricted to the first $r$ lines.
The theorem follows immediately from \eqref{eq:main-inductive-bound}, since Wasserstein
distance is bounded above by the $L^p$-distance under any coupling.

The case $r=0$ is trivial ($T$ contains only input, hence deterministic, $\sH$-vars). 
Assume that the coupling has been constructed up to line
$r-1$ and that \eqref{eq:main-inductive-bound} holds at level $r-1$. We distinguish the
two possible operations at line $r$.

\noindent \emph{The $\nonlin$ case.}
Assume that line $r$ is
\[
\sh^{(r)}=\phi(\sh_1,\ldots,\sh_k),
\]
where $\sh_1,\ldots,\sh_k:\sH(n_r)$ are previously defined $\sH$-vars. We extend the
coupling by setting
\[
h^{(r)}=\phi(h_1,\ldots,h_k),
\qquad
\bar h^{(r)}=\phi(\bar h_1,\ldots,\bar h_k).
\]
Since $\phi$ is Lipschitz and acts componentwise, we have pointwise
\[
\|h^{(r)}-\bar h^{(r)}\|
\le
\Lip(\phi)
\left(
\sum_{j=1}^k \|h_j-\bar h_j\|^2
\right)^{1/2}.
\]
Dividing by $\sqrt{n_r}$ gives
\[
\left\|
\frac{h^{(r)}-\bar h^{(r)}}{\sqrt{n_r}}
\right\|
\le
\Lip(\phi)
\left(
\sum_{j=1}^k
\left\|
\frac{h_j-\bar h_j}{\sqrt{n_r}}
\right\|^2
\right)^{1/2}.
\]
Therefore,
\[
\left\|
\frac{h^{(r)}-\bar h^{(r)}}{\sqrt{n_r}}
\right\|_{L^p}
\le
\Lip(\phi)
\left\|
\left(
\frac{h_j-\bar h_j}{\sqrt{n_r}}
\right)_{j=1,\ldots,k}
\right\|_{L^p}.
\]
Since $\sh_1,\ldots,\sh_k$ are previous $\sH$-vars of the same shape $n_r$, the last
term is bounded by the induction hypothesis:
\[
\left\|
\frac{h^{(r)}-\bar h^{(r)}}{\sqrt{n_r}}
\right\|_{L^p}
\le
\scc
\sum_{m\in\nhidd(r-1)}\frac1{\sqrt{m}}.
\]
Combining this estimate with the induction hypothesis for the previous coordinates yields
\[
\|\Delta_r\|_{L^p}
\le
\scc
\sum_{m\in\nhidd(r)}\frac1{\sqrt{m}},
\]
noticing that $\nhidd(r-1) = \nhidd(r)$.

\noindent \emph{The $\matmul$ case.}
Assume that line $r$ is
\[
\sg=\sW\sh,
\qquad
\sh:\sH(m_r),
\qquad
\sW:\sA(n_r,m_r),
\]
where $n_r$ and $m_r$ do not necessarily belong to $\nhidd(r-1)$, because $n_r$ may be a completely new width, and $m_r$ could be an input dimension or an output dimension for the restricted program up to line $r-1$. In the former case (as well as in the case that $m_r \in \nhidd(r-1)$ already)  it holds $\nhidd(r) =  \nhidd(r-1)$, while in the latter we have $\nhidd(r) =\nhidd(r-1)\cup \{m_r\}$. Let $g$ and $\bar g$ be the finite-width and infinite-width executions of $\sg$. Conditionally on the past lines, both are Gaussian:
\[
g\mid \cF_{r-1}
\sim
\mathcal N\!\left(\mu_{g\mid \cF_{r-1}},\Sigma_{g\mid \cF_{r-1}}\right),
\]
and
\[
\bar g\mid \bF_{r-1}
\sim
\mathcal N\!\left(\mu_{\bar g\mid \bF_{r-1}},
\Sigma_{\bar g\mid \bF_{r-1}}\right).
\]
The conditional means and covariances are given by the structure lemmas for the finite-width
and infinite-width executions.

We extend the coupling by taking a standard Gaussian vector
$Z\sim\mathcal N(0,\Id_{n_r})$, independent of the past, and setting
\begin{equation}
\label{eq:matmul-coupling}
h^{(r)}=g
:=
\mu_{g\mid \cF_{r-1}}
+
\sqrt{\Sigma_{g\mid \cF_{r-1}}}\,Z,
\qquad
\bar h^{(r)}=\bar g
:=
\mu_{\bar g\mid \bF_{r-1}}
+
\sqrt{\Sigma_{\bar g\mid \bF_{r-1}}}\,Z.
\end{equation}
This gives the correct conditional laws, hence the correct marginal laws up to line $r$.

We now prove the two stability estimates
\begin{align}
\label{eq:mean-stability-main}
\left\|
\frac{
\mu_{g\mid \cF_{r-1}}-\mu_{\bar g\mid \bF_{r-1}}
}{\sqrt{n_r}}
\right\|_{L^p}
&\le
\scc \sum_{m \in \nhidd(r)}\frac1{\sqrt{m}},
\\
\label{eq:cov-stability-main}
\left\|
\sqrt{\Sigma_{g\mid \cF_{r-1}}}
-
\sqrt{\Sigma_{\bar g\mid \bF_{r-1}}}
\right\|_{L^p}
&\le
\scc \sqrt{n_r}
\sum_{m \in \nhidd(r)}\frac1{\sqrt{m}}.
\end{align}

Assuming \eqref{eq:mean-stability-main}-\eqref{eq:cov-stability-main} for the moment,
we conclude the $\matmul$ step. Let
\[
A_r
:=
\sqrt{\Sigma_{g\mid \cF_{r-1}}}
-
\sqrt{\Sigma_{\bar g\mid \bF_{r-1}}}.
\]
By \eqref{eq:matmul-coupling},
\[
\frac{\|g-\bar g\|}{\sqrt{n_r}}
\le
\frac{\|\mu_{g\mid \cF_{r-1}}-\mu_{\bar g\mid \bF_{r-1}}\|}{\sqrt{n_r}}
+
\frac{\|A_r Z\|}{\sqrt{n_r}}.
\]
Since $Z$ is standard Gaussian and independent of the past, the Gaussian moment bound
gives
\[
\|A_r Z\|_{L^p}
\le
\scc_p \|A_r\|_{L^p}.
\]
Therefore, by \eqref{eq:mean-stability-main} and \eqref{eq:cov-stability-main},
\[
\left\|
\frac{g-\bar g}{\sqrt{n_r}}
\right\|_{L^p}
\le
\scc \sum_{m \in \nhidd(r)}\frac1{\sqrt{m}}.
\]
Combining this with the induction hypothesis gives
\[
\|\Delta_r\|_{L^p}
\le
\|\Delta_{r-1}\|_{L^p}
+
\left\|
\frac{g-\bar g}{\sqrt{n_r}}
\right\|_{L^p}
\le
\scc \sum_{m \in \nhidd(r)}\frac1{\sqrt{m}}.
\]

It remains to prove the two stability estimates. We first dispose of the case
$m_r\in\nin$. In this case the parent $\sh:\sH(m_r)$ is an input variable, or a $\nonlin$ transformation of them,  hence its finite-width and infinite-width executions coincide deterministically. Consequently the
finite-width and infinite-width Gram data entering the structure formulas are identical.
The conditional means and covariances of $g$ and $\bar g$ therefore coincide, and the
left-hand sides of both stability estimates vanish. Hence, we may assume from now on that
$m_r\notin\nin$. We prove separately the case of mean and covariance.

It remains to prove the two stability estimates. We first consider the case
\(m_r\in\nin\). In the basic $\netsor$ language, the current parent
\(\sh:\sH(m_r)\) is obtained from deterministic input variables by ordinary
$\nonlin$ operations. Hence its finite-width and infinite-width executions
coincide deterministically, and no new kernel LLN over the input dimension \(m_r\)
is required. The Gram entries involving the current parent are structural.

However, if the same \(\sA\)-var \(\sW\) has already been used earlier in the
program, the conditional means and covariances need not coincide exactly: the only
possible discrepancy comes from previous uses of \(\sW\), and is already controlled
by the induction hypothesis and by the kernel estimates for the previous hidden
parents. Therefore the estimates below remain valid with
\[
\nhidd(r)=\nhidd(r-1).
\]
We may then continue with the same mean and covariance stability arguments, noting
that no additional \(m_r^{-1/2}\) contribution is generated when \(m_r\in\nin\).

\noindent \emph{Covariance stability.}
By the structure lemmas, there exist scalar quantities $a_r,\bar a_r\ge 0$ such that
\[
\Sigma_{g\mid \cF_{r-1}} = a_r\,\Id_{n_r},
\qquad
\Sigma_{\bar g\mid \bF_{r-1}} = \bar a_r\,\Id_{n_r}.
\]
Moreover, $a_r$ and $\bar a_r$ are obtained by applying the same finite-dimensional map
to the finite-width and infinite-width Gram data associated with the previous parents of the weight
$\sW$. More precisely,
\[
a_r
=
f\!\left(
\frac{1}{m_r} H_{\sW}^{(r-1),\top}H_{\sW}^{(r-1)}
\right),
\qquad
\bar a_r
=
f\!\left(
\bH_{\sW}^{(r-1)}\ldot \bH_{\sW}^{(r-1)}
\right),
\]
for a map $f$ depending only on the program. 
In particular, by the structure lemmas, $f(X)$ is exactly the Schur complement $X/Y$, 
where $Y$ is the square submatrix of $X$ obtained by removing the last row and column.
In our case, when $X = \frac1m H_{\sW}^{(r-1),\top}H_{\sW}^{(r-1)}$, the submatrix $Y$ is given by $\frac1m H_{\sW}^{(r-2),\top}H_{\sW}^{(r-2)}$
and when $X = \bH_{\sW}^{(r-1)}\ldot \bH_{\sW}^{(r-1)}$, $Y$ is given by $\bH_{\sW}^{(r-2)}\ldot \bH_{\sW}^{(r-2)}$.

Since the program is non-degenerate in
infinite-width execution, the map $x\mapsto \sqrt{f(x)}$ is differentiable at the infinite-width
Gram point. Lemma~\ref{lem:concentration}, applied to this map, together with the kernel
LLN for the truncated program up to line $r-1$, yields
\[
\|\sqrt{a_r}-\sqrt{\bar a_r}\|_{L^p}
\le
\scc \sum_{m \in \nhidd(r-1)\cup \{m_r\}}\frac1{\sqrt{m}} = \scc\sum_{m \in \nhidd(r)}\frac1{\sqrt{m}},
\]
since $m_r$ is either a hidden width or an output dimension for the restricted program.
Moreover,
\[
\sqrt{\Sigma_{g\mid \cF_{r-1}}}
-
\sqrt{\Sigma_{\bar g\mid \bF_{r-1}}}
=
(\sqrt{a_r}-\sqrt{\bar a_r})\,\Id_{n_r},
\]
we obtain
\[
\left\|
\sqrt{\Sigma_{g\mid \cF_{r-1}}}
-
\sqrt{\Sigma_{\bar g\mid \bF_{r-1}}}
\right\|_{L^p}
=
\sqrt{n_r}\,
\|\sqrt{a_r}-\sqrt{\bar a_r}\|_{L^p}
\le
\scc\sqrt{n_r}
\sum_{m \in \nhidd(r)}\frac1{\sqrt{m}},
\]
which proves \eqref{eq:cov-stability-main}.

\noindent \emph{Mean stability.}
We prove a slightly stronger estimate. Fix the parent $\sh:\sH(m)$ appearing in line
$r$, and write $h,\bar h$ for its finite-width and infinite-width executions. For every
$s\le r-1$, we claim that
\begin{equation}
\label{eq:mean-inner-induction}
\left\|
\frac{
\mu_{W\mid \cF_s}h
-
\bar\mu_{W\mid \bF_s}\bar h
}{\sqrt{n_r}}
\right\|_{L^p}
\le
\scc\sum_{m \in \nhidd(r)}\frac1{\sqrt{m}}.
\end{equation}
Taking $s=r-1$ gives \eqref{eq:mean-stability-main}.

We prove \eqref{eq:mean-inner-induction} by induction over $s$. For $s=0$, both
conditional means are zero. Assume the estimate holds at level $s-1$. If line $s$ is not
a $\matmul$ line using the weight $\sW$, then the conditional mean of $\sW$ does
not change and there is nothing to prove. Thus suppose line $s$ is
\[
\sg'=\sW\sh',
\]
with executions $(g',h')$ and $(\bar g',\bar h')$.

The update formula for the conditional mean gives
\[
\mu_{W\mid \cF_s}h
=
\mu_{W\mid \cF_{s-1}}h
+
Z_s u_s,
\]
where
\[
Z_s
:=
\big(g'-\mu_{W\mid \cF_{s-1}}h'\big)
\sqrt{(h'^{\top}R_{W\mid \cF_{s-1}}h'\big)^{+} }
\in\R^{n_r},
\]
and
\[
u_s
:=
\sqrt{\big(h'^{\top}R_{W\mid \cF_{s-1}}h'\big)^{+}}
h'^{\top}R_{W\mid \cF_{s-1}}h
\in\R.
\]
Similarly,
\[
\bar\mu_{W\mid \bF_s}\bar h
=
\bar\mu_{W\mid \bF_{s-1}}\bar h
+
\bar Z_s\bar u_s,
\]
with the analogous barred definitions. By the coupling construction at line $s$, the
standardized residuals are driven by the same Gaussian noise, hence
\[
Z_s=\bar Z_s
\]
under our coupling. Moreover,
\[
Z_s\mid \cF_{s-1}
\sim
\mathcal N\!\left(0,\frac1m_r\Id_{n_r}\right).
\]
Therefore,
\[
\left\|
\frac{
\mu_{W\mid \cF_s}h-\bar\mu_{W\mid \bF_s}\bar h
}{\sqrt{n_r}}
\right\|_{L^p}
\le
\left\|
\frac{
\mu_{W\mid \cF_{s-1}}h-\bar\mu_{W\mid \bF_{s-1}}\bar h
}{\sqrt{n_r}}
\right\|_{L^p}
+
\frac1{\sqrt{n_r}}\|Z_s(u_s-\bar u_s)\|_{L^p}.
\]
By Hölder's inequality and Gaussian moment estimates,
\[
\frac1{\sqrt{n_r}}\|Z_s(u_s-\bar u_s)\|_{L^p}
\le
\frac1{\sqrt{n_r}}\|Z_s\|_{L^{2p}}\|u_s-\bar u_s\|_{L^{2p}}
\le
\frac{\scc}{\sqrt m_r}\|u_s-\bar u_s\|_{L^{2p}}.
\]

It remains to control $u_s-\bar u_s$. By the projector representation of the right
covariance, both $u_s$ and $\bar u_s$ are obtained from the corresponding Gram matrices via
the same finite-dimensional map:
\[
u_s
=
f_s\!\left(
\frac{1}{m_r} H_{\sW}^{(s-1),\top}H_{\sW}^{(s-1)}
\right),
\qquad
\bar u_s
=
f_s\!\left(
\bH_{\sW}^{(s-1)}\ldot \bH_{\sW}^{(s-1)}
\right).
\]
By non-degeneracy, $f_s$ is differentiable at the infinite-width Gram point. Moreover,
\[
|u_s|
=
\left|
\sqrt{\big(h'^{\top}R_{W\mid \cF_{s-1}}h'\big)^{+}}
h'^{\top}R_{W\mid \cF_{s-1}}h
\right|
\le
\|R_{W\mid \cF_{s-1}}^{1/2}h\|
\le
\|h\|,
\]
since $\|R_{W\mid \cF_{s-1}}\|_{\op}\le 1$, and the same bound holds for $\bar u_s$.
Thus the uniform moment bounds imply
\[
\|u_s\|_{L^{4p}}+\|\bar u_s\|_{L^{4p}}
\le
\scc\sqrt{m_r}.
\]
Applying Lemma~\ref{lem:concentration} to $f_s$, together with the kernel LLN for the
truncated program, yields
\[
\|u_s-\bar u_s\|_{L^{2p}}
\le
\scc \sqrt{m_r}
\sum_{m\in\nhidd(r-1)\cup\{m_r\}}\frac1{\sqrt{m}} = \scc \sqrt{m_r}
\sum_{m\in\nhidd(r)}\frac1{\sqrt{m}}.
\]
Consequently,
\[
\frac1{\sqrt{n_r}}\|Z_s(u_s-\bar u_s)\|_{L^p}
\le
\scc \sum_{m\in\nhidd(r)}\frac1{\sqrt{m}}.
\]
Iterating over $s\le r-1$ proves
\eqref{eq:mean-inner-induction}. Hence \eqref{eq:mean-stability-main} follows.

The $\nonlin$ and $\matmul$ cases close the induction. Taking $r=L$ in
\eqref{eq:main-inductive-bound} gives
\[
\left\|
\left(
\frac{h^{(\ell)}-\bar h^{(\ell)}}{\sqrt{n_\ell}}
\right)_{\ell=1,\ldots,L}
\right\|_{L^p}
\le
\scc \sum_{m\in\nhidd}\frac1{\sqrt{m}}.
\]
Therefore,
\[
\W_p\left(
\left(\frac{h^{(\ell)}}{\sqrt{n_\ell}}\right)_{\ell=1}^L,
\left(\frac{\bar h^{(\ell)}}{\sqrt{n_\ell}}\right)_{\ell=1}^L
\right)
\le
\scc \sum_{m\in\nhidd}\frac1{\sqrt{m}},
\]
as claimed.

\subsection{Proof of Theorem~\ref{thm:main-netsork}}
\label{sec:proof-netsork}

The proof follows the line-by-line coupling used for Theorem~\ref{thm:main}, enlarged to
include the scalar variables. 


The $\matmul$ case is governed by the same conditional Gaussian regression
formulas as in the basic $\netsor$ proof. Indeed, scalar variables are measurable functions
of the previously generated variables and do not introduce additional observations of the
weights. Thus, conditionally on the previous lines, the finite-width execution $g$ of a $\matmul$ line, \(\sg=\sW \sh\), is still
Gaussian with conditional mean and covariance given by the structure formulas of
Section~\ref{sec:weight-conditional-laws}.

There is, however, one minor difference with respect to the basic $\netsor$ setting. In a
$\netsork$ program, an \(H\)-var with structural input dimension $m\in \nin$ need not be one of the
deterministic input variables: it may have been obtained from input variables through a
$\nonlinpl$ instruction depending on scalar variables. Such a variable is therefore not
necessarily deterministic in the finite-width execution, as opposed to the variables obtained from $\nonlin$ instructions applied to input variables. Its infinite-width execution is in any case deterministic, and its finite-width fluctuation is
already controlled by the induction hypothesis on the preceding $\nonlinpl$ and
$\scal$ lines. Consequently, no new kernel LLN over an input dimension is required.

For variables of hidden dimension $m\in \nhidd$, the argument is identical to the $\netsor$ case. The local Lipschitz assumption on $\netsork$ programs ensures that the
finite-dimensional maps appearing in the structure formulas are differentiable at the
infinite-width Gram point.

Assuming that the $r$-th line is a $\kernel$ line $\ss=\sh^\star\sh'$, with $\sh,\sh' \colon \sH(m_r)$, the two executions are
\[
s=\frac{1}{m_r} h^\top h',
\qquad
\bar s=\bar h\ldot\bar h'.
\]
Thus the estimate
\[
\|s-\bar s\|_{L^p}
\le
\scc\sum_{m\in\nhidd(r)}\frac1{\sqrt{m}}
\]
is exactly Corollary \ref{coro:kernel-lln}, applied to the truncated program up to the current line.
Take into account that $m_r$ belongs to $\nhidd(r)$, as opposed to the case of $\netsor$.

For a $\scal$ line,
\[
\ss=\psi(\ss_1,\ldots,\ss_u),
\]
the estimate follows from the finite-dimensional local Lipschitz property of $\psi$,
the induction hypothesis for the previous scalar variables, and the moment bounds for
scalar variables, which are obtained inductively from the kernel LLN and the growth
assumptions.

It remains to discuss a $\nonlinpl$ line,
\[
\sh=\phi(\sh_1,\ldots,\sh_k;\ss_1,\ldots,\ss_u):\sH(m_r).
\]
Let $h,\bar h$ be the corresponding executions. Applying
\eqref{eq:separated-linear-lip} componentwise gives
\[\begin{split}
|h_a-\bar h_a|
& \le
C (1+\nor{s}+\nor{\bar s}) \bra{
\sum_{i=1}^k |h_{i,a}-\bar h_{i,a}|^2}^{1/2}
\\
& \quad + C \left(1+\bra{\sum_{i=1}^k |h_{i,a}|^2}^{1/2}+\bra{\sum_{i=1}^k|\bar h_{i,a}|^2 }^{1/2} \right)
|s-\bar s|,
\end{split}
\]
where $s=(s_1,\ldots,s_u)$ and $\bar s=(\bar s_1,\ldots,\bar s_u)$. Hence, squaring and summing over $a=1,\ldots, m_r$, we obtain eventually
\[
\|h-\bar h\| \le
\scc(1+\nor{s}+\nor{\bar s})
\sum_{i=1}^k \nor{ h_i-\bar h_i }
+
\scc \nor{s-\bar s}
\left(
1+\sum_{i=1}^k \nor{ h_i }+ \nor{\bar h_i} 
\right).
\]
Dividing both sides by $\sqrt{m_r}$ and taking $L^p$ norms and using Hölder's inequality, the induction hypothesis controls the
differences, while the normalized moment bounds for the previously constructed
$\sH$-vars and $\sC$-vars control the coefficients. This yields
\[
\left\|\frac{h-\bar h}{\sqrt m_r}\right\|_{L^p}
\le
\scc\sum_{m \in \nhidd}\frac1{\sqrt{m}},
\]
and induction closes as in the proof of Theorem~\ref{thm:main}.

\section{Numerical Experiments}
\label{sec:numerical_experiments}

In this section, we validate the quantitative wide-limit convergence of finite-width random neural networks toward their infinite-width limits. We focus on the convergence of the network's final output law via an \emph{ensemble sampling} paradigm. 
For fully connected architectures, this output-law experiment is designed to reveal the
sharper \(\mathcal O(n^{-1})\) convergence rate known under non-degeneracy hypotheses
\cite{trevisan2023wide, Favaro2025}. For the other architectures, the experiments should be interpreted as
evidence of the finite-width convergence predicted by the present theorem, rather than as
a theorem-level claim of an \(\mathcal O(n^{-1})\) rate.

\subsection{Experimental Setup and Architectures}
We evaluate a suite of architectures of increasing complexity, corresponding to the examples introduced in Section \ref{subsec:program-graphs}, to illustrate the finite-width convergence predicted by Theorem \ref{thm:main} across several
architectures. In all of the following experiments, the infinite-width limit moments are computed using deterministic Gauss-Hermite quadrature to avoid integration noise.

\paragraph{Multi-Layer Perceptrons:} 
We evaluate both shallow and deep Multi-Layer Perceptrons using $\tanh$ and ReLU activations, mirroring the structures of Algorithm \ref{alg:shallow} and Algorithm \ref{alg:deep}, respectively. These setups capture standard layer-wise transitions where separate, independent weight matrices are sampled per layer. The deep MLP experiment ($L=4$) demonstrates how compounding non-linearities propagate through the infinite-width recursion while maintaining the finite-width convergence limits.

\paragraph{Time-unrolled Recurrent Neural Networks:} 
We test recurrent networks following the tied-weight architecture of Algorithm \ref{alg:rnn}. In these experiments, the network is unrolled over $T=5$ time steps with a $\tanh$ activation to demonstrate convergence under weight-sharing constraints. Because the recurrent matrix $\sW^h$ is shared across time, the infinite-width computation must track cross-time second moments $K_{t,s} = \mathbb{E}[h_t h_s^\top]$. The limiting marginal laws at any given time step $t$ are then extracted from the diagonal blocks $K_{t,t}$.

\paragraph{Residual Networks:} 
To further validate the framework, we evaluate deep networks with skip connections, formulated as $\sh^{(\ell)} = \sh^{(\ell-1)} + \phi(\sW^{(\ell)} \sh^{(\ell-1)})$ as detailed in Algorithm \ref{alg:resnet}, where $\phi$ is the ReLU activation. Because a new, independent weight matrix $\sW^{(\ell)}$ is drawn on each residual update, the samples for the current state and its update can be treated as independent when constructing marginal samples for the subsequent layer.

\begin{algorithm}[htpb]
\floatname{algorithm}{Netsor program}
\caption{Residual Network}\label{alg:resnet}
\begin{algorithmic}
\Require $\sx: \sH(d)$
\Require $\sW^{(0)}: \sA(n, d)$
\Require $\sW^{(\ell)}: \sA(n, n)$ for $\ell = 1, \dots, L$
\State \algline $\sg^{(0)} = \sW^{(0)} \sx: \sG(n)$ \Comment{\matmul}
\State \algline $\sh^{(0)} = \sg^{(0)}: \sH(n)$ \Comment{\nonlin}
\For{$\ell =1, \ldots, L$}
  \State \algline $\sg^{(\ell)} = \sW^{(\ell)} \sh^{(\ell-1)}: \sG(n)$ \Comment{\matmul}
  \State \algline $\sh^{(\ell)} = \sh^{(\ell-1)} + \phi( \sg^{(\ell)} ): \sH(n)$ \Comment{\nonlin}
\EndFor
\end{algorithmic}
\end{algorithm}

\subsection{Output-Law Convergence and Estimation Noise}

To quantify the convergence of the output distributions, we utilize the Sliced Wasserstein-1 distance. For probability measures $\mu$ and $\nu$ on $\mathbb{R}^M$, it is defined as:
\begin{equation}
    \mathcal{SW}_1(\mu,\nu) = \int_{\mathbb{S}^{M-1}} \mathcal{W}_1(\theta_\sharp \mu, \theta_\sharp \nu) \, d\sigma(\theta),
\end{equation}
where $\theta_\sharp$ denotes the pushforward measure onto the direction $\theta \in \mathbb{S}^{M-1}$, and $\sigma$ is the uniform measure on the unit sphere. The bound $\mathcal{SW}_1 \leq \mathcal{W}_1$ holds inherently because the orthogonal projection onto $\theta$ is a $1$-Lipschitz map. Moreover, $\mathcal{SW}_1$ circumvents the computational complexity of $\W_1$ associated with the curse of dimensionality, making it preferred for empirical evaluations.

For each architectural width $n$, we initialize an ensemble of $N = 5000$ independent networks and extract the scalar readouts. By evaluating these readouts across independent seeds, we decouple the empirical sample size from the network width $n$. The $\mathcal{SW}_1$ distance is computed between these $N$ outputs and the analytical distribution of the activations of the limiting Gaussian process.

For an empirical measure $\hat{\mu}_N$ constructed from $N$ i.i.d. samples of a distribution $\mu$, Theorem 1 in \cite{Fournier2015} yields $\mathbb{E}[\mathcal{SW}_1(\mu,\hat{\mu}_N)] \le c_1 N^{-1/2}$. Let $\mu_n$ denote the exact finite-width distribution of the network activations, $\hat{\mu}_{n,N}$ denote the empirical measure formed by the $N$ ensemble samples, and $\mu_\infty$ denote the limiting distribution of the Gaussian process activations. By the triangle inequality, the empirical distance we observe is bounded by the true architectural convergence plus this statistical noise:
\begin{equation}
    \mathcal{SW}_1(\hat{\mu}_{n,N}, \mu_\infty) \le \mathcal{SW}_1(\mu_n, \mu_\infty) + \frac{c_1}{\sqrt{N}} \le \frac{c_2}{n} + \frac{c_1}{\sqrt{N}},
\end{equation}
where we substitute the expected theoretical rate $\mathcal{O}(n^{-1})$ from \cite{Favaro2025,trevisan2023wide} for MLP architectures. The sampling error $N^{-1/2}$ does not dominate the architectural signal for network widths up to approximately $n \le \sqrt{N} \approx 70$; for larger $n$, the observed slope is expected to flatten toward $-1/2$ as the estimation noise becomes the leading term. This behavior is consistent with the fitted slopes reported in Table~\ref{tab:numerical-slopes}, which lie strictly between $-1/2$ and $-1$ across all architectures, reflecting a mixture of the true convergence signal and the sampling floor.

In Figure~\ref{fig:output_law_raw} we plot the Sliced Wasserstein results for our four experiments in log-log scale.

For each architecture, we estimate the empirical convergence exponent by a least-squares
fit of
\[
\log d_n = \alpha + \beta \log n,
\]
where \(d_n\) denotes the \(\mathcal{SW}_1\) distance between the finite-width execution at width $n$ and the infinite-width execution of each selected architecture.
The fitted value of \(\beta\) is reported in Table~\ref{tab:numerical-slopes}. Slopes strictly below
\(-1/2\) indicate that the output-law convergence is visible beyond the sampling noise floor.

\begin{figure}[htpb]
    \centering
    \begin{subfigure}[b]{0.48\textwidth}
        \centering
        \includegraphics[width=\textwidth]{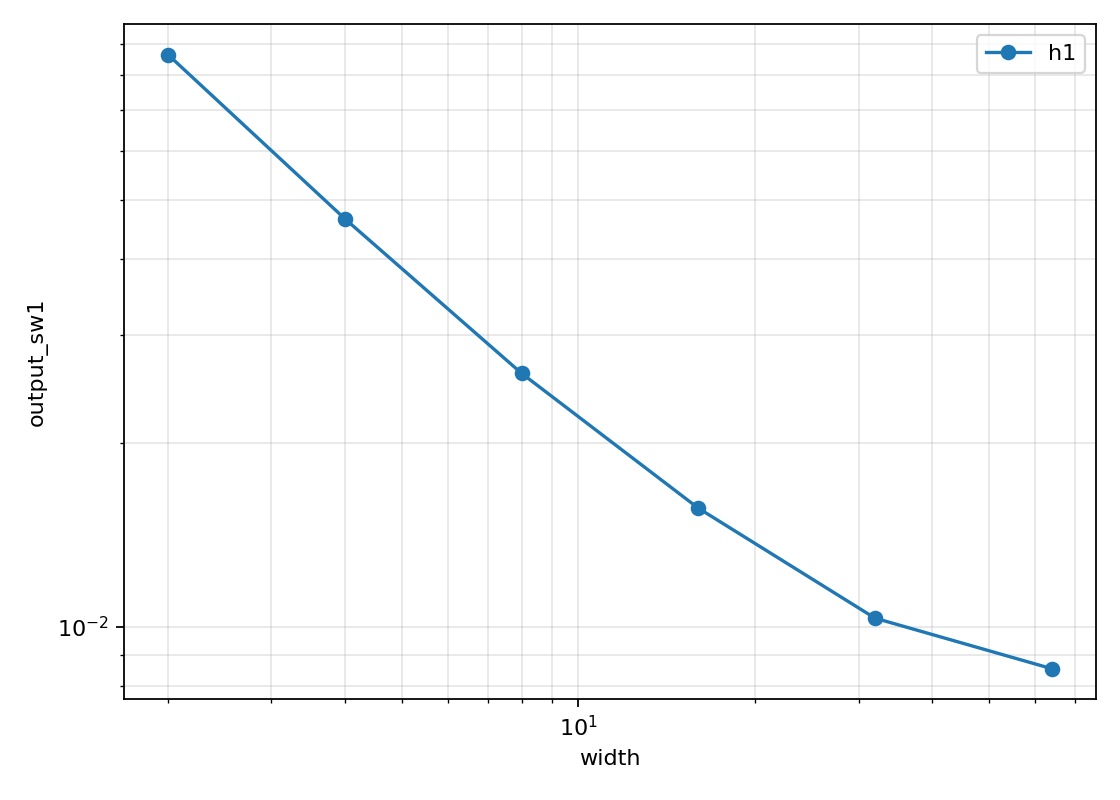}
        \caption{Shallow MLP}
    \end{subfigure}
    \hfill
    \begin{subfigure}[b]{0.48\textwidth}
        \centering
        \includegraphics[width=\textwidth]{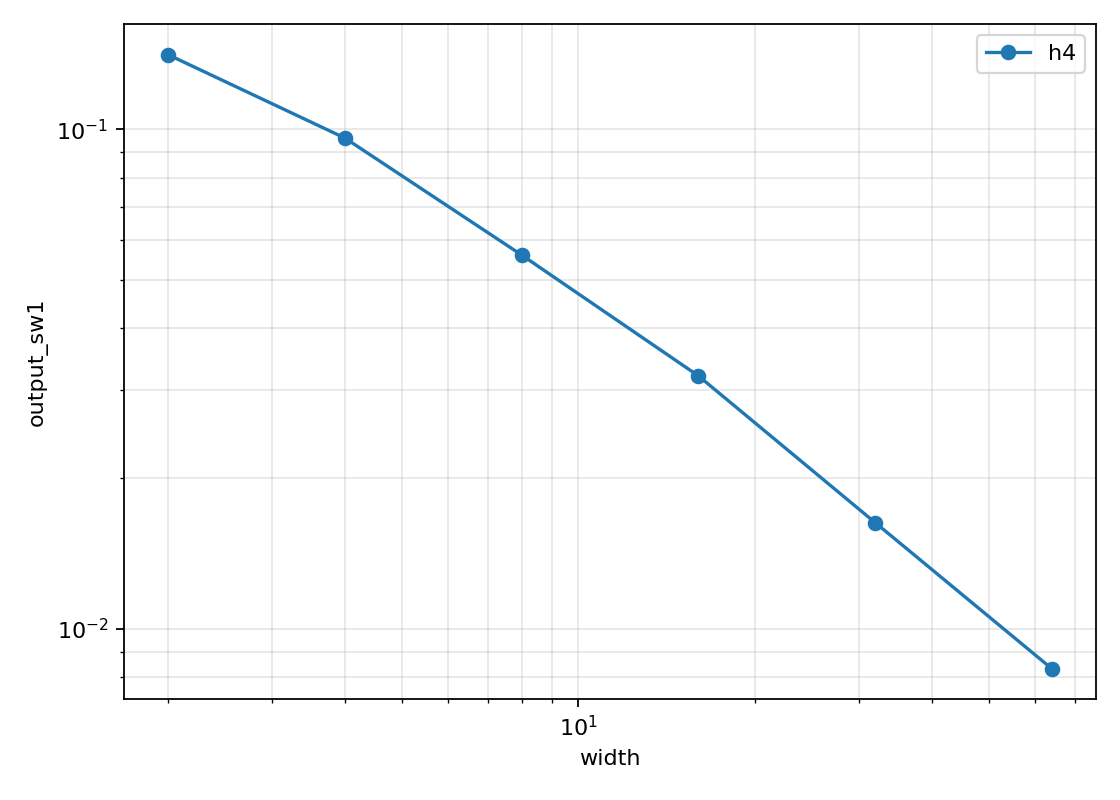}
        \caption{Deep MLP ($L=4$)}
    \end{subfigure}
    
    \vspace{0.4cm}
    
    \begin{subfigure}[b]{0.48\textwidth}
        \centering
        \includegraphics[width=\textwidth]{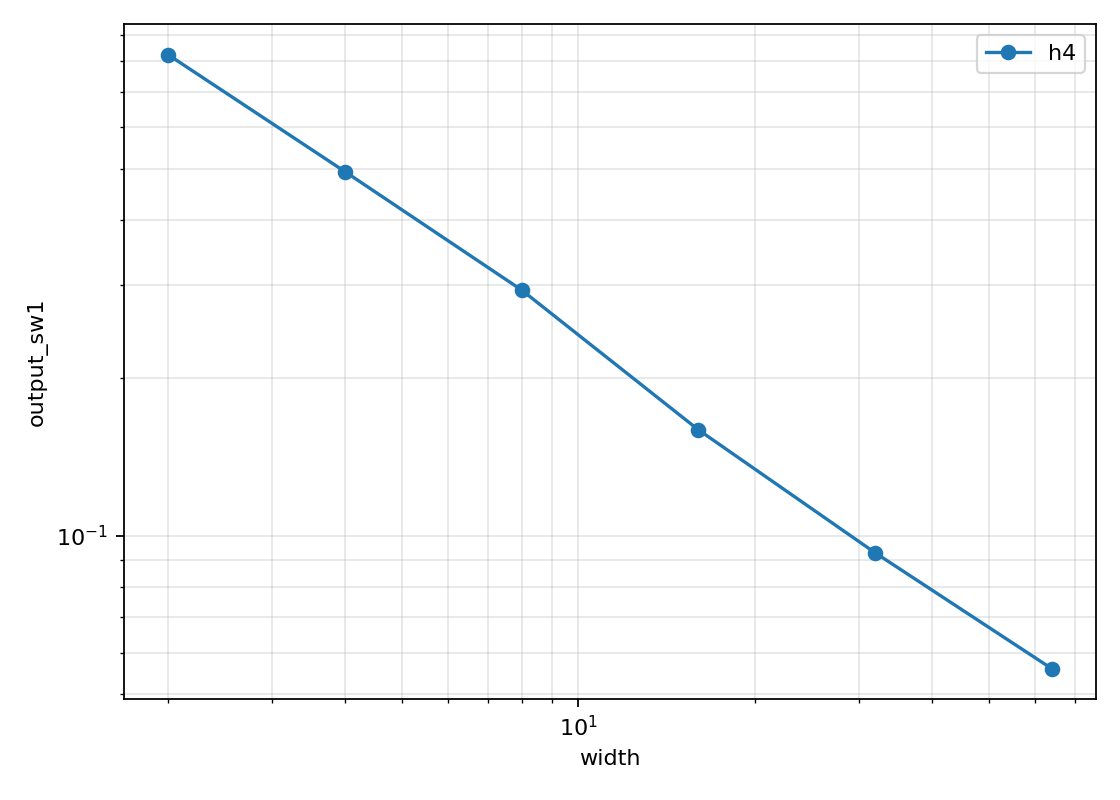}
        \caption{Residual Network}
    \end{subfigure}
    \hfill
    \begin{subfigure}[b]{0.48\textwidth}
        \centering
        \includegraphics[width=\textwidth]{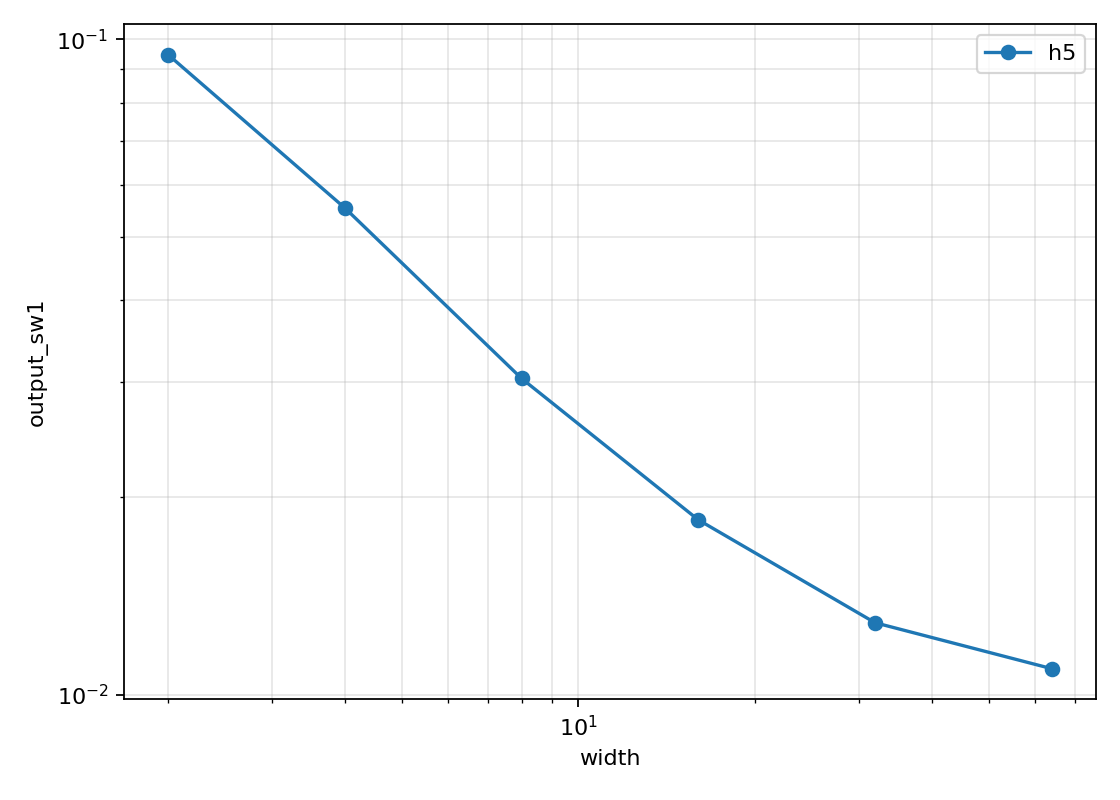}
        \caption{Time-unrolled RNN}
    \end{subfigure}
    \caption{Sliced Wasserstein-1 ($\mathcal{SW}_1$) distances between the empirical output laws ($N=5000$ ensemble samples) and the law of the limiting Gaussian process activations across four architectures, plotted in log-log scale.}
    \label{fig:output_law_raw}
\end{figure}

\begin{table}[h]
\centering
\begin{tabular}{lc}
\hline
Architecture & \(\mathcal{SW}_1\) slope \\
\hline
Shallow MLP & \(-0.721\) \\
Deep MLP \(L=4\) & \(-0.916\) \\
Residual network & \(-0.797\) \\
Time-unrolled RNN & \(-0.579\) \\
\hline
\end{tabular}
\caption{Least-squares log-log slopes for the $\mathcal{SW}_1$ curves shown in
Figure~\ref{fig:output_law_raw}. All slopes lie strictly below $-1/2$, confirming
that the architectural convergence signal is visible above the empirical sampling floor.}
\label{tab:numerical-slopes}
\end{table}

\printbibliography

\end{document}